\documentclass{article}

\usepackage{etoolbox}
\newtoggle{icml}
\togglefalse{icml}
\newcommand{\icml}[1]{\iftoggle{icml}{#1}{}}
\newcommand{\arxiv}[1]{\iftoggle{icml}{}{#1}}

\arxiv{

\usepackage{verbatim}
\usepackage{setspace}
\usepackage{longtable}
\usepackage{makecell} 

\setlength{\parindent}{0pt}
\setlength{\parskip}{6pt}
\usepackage{titlesec}
\titlespacing{\section}{0pt}{10pt}{0pt}
\titlespacing{\subsection}{0pt}{\parskip}{0pt}
\titlespacing{\subsubsection}{0pt}{\parskip}{0pt}
\usepackage[letterpaper, margin=1in]{geometry}	

\usepackage{xpatch}
\xpatchcmd{\proof}{\itshape}{\normalfont\proofnameformat}{}{}
\newcommand{\proofnameformat}{\bfseries}

\makeatletter
\renewcommand\paragraph{\@startsection{paragraph}{4}{\z@}%
                                    {2pt \@plus0pt \@minus0pt}%
                                    {-1em}%
                                    {\normalfont\normalsize\bfseries}}
\makeatother

\usepackage[dvipsnames]{xcolor}
\usepackage{microtype}

\usepackage{algorithm}
\usepackage[noend]{algpseudocode}

\usepackage[authoryear]{natbib}
\bibliographystyle{abbrvnat}
\bibpunct{(}{)}{;}{a}{,}{,}

\usepackage{amsthm}
\usepackage{thmtools}

\usepackage{mathtools}
\usepackage{amsmath}
\usepackage{bbm}
\usepackage{amsfonts}
\usepackage{amssymb}

}

\usepackage{microtype}
\usepackage{graphicx}
\usepackage{subfigure}
\usepackage{booktabs} %

\icml{
\usepackage[urlcolor=black]{hyperref}

\usepackage[scaled=.90]{helvet}

\usepackage[usenames,dvipsnames]{xcolor}
}

\icml{
  \usepackage[accepted]{icml2020}
}

\usepackage{amsmath,amsthm,amssymb,bbm}
\usepackage{breqn}
\usepackage{changepage}
\usepackage{dsfont}
\usepackage{bm}
\usepackage{mathrsfs}
\usepackage{natbib}
\usepackage{mathtools}
\usepackage{amsfonts}
\usepackage{inconsolata}

\usepackage{xspace}
\usepackage{enumitem}
\usepackage{verbatim}

\DeclarePairedDelimiter{\abs}{\lvert}{\rvert} %
\DeclarePairedDelimiter{\brk}{[}{]}
\DeclarePairedDelimiter{\crl}{\{}{\}}
\DeclarePairedDelimiter{\prn}{(}{)}

\DeclarePairedDelimiter{\dtri}{\llangle}{\rrangle}

\DeclareMathOperator{\En}{\mathbb{E}}

\newcommand{\wh}[1]{\widehat{#1}}

\def\ddefloop#1{\ifx\ddefloop#1\else\ddef{#1}\expandafter\ddefloop\fi}
\def\ddef#1{\expandafter\def\csname bb#1\endcsname{\ensuremath{\mathbb{#1}}}}
\ddefloop ABCDEFGHIJKLMNOPQRSTUVWXYZ\ddefloop
\def\ddefloop#1{\ifx\ddefloop#1\else\ddef{#1}\expandafter\ddefloop\fi}
\def\ddef#1{\expandafter\def\csname b#1\endcsname{\ensuremath{\mathbf{#1}}}}
\ddefloop ABCDEFGHIJKLMNOPQRSTUVWXYZ\ddefloop
\def\ddef#1{\expandafter\def\csname sf#1\endcsname{\ensuremath{\mathsf{#1}}}}
\ddefloop ABCDEFGHIJKLMNOPQRSTUVWXYZ\ddefloop
\def\ddef#1{\expandafter\def\csname c#1\endcsname{\ensuremath{\mathcal{#1}}}}
\ddefloop ABCDEFGHIJKLMNOPQRSTUVWXYZ\ddefloop
\def\ddef#1{\expandafter\def\csname h#1\endcsname{\ensuremath{\widehat{#1}}}}
\ddefloop ABCDEFGHIJKLMNOPQRSTUVWXYZ\ddefloop
\def\ddef#1{\expandafter\def\csname hc#1\endcsname{\ensuremath{\widehat{\mathcal{#1}}}}}
\ddefloop ABCDEFGHIJKLMNOPQRSTUVWXYZ\ddefloop
\def\ddef#1{\expandafter\def\csname t#1\endcsname{\ensuremath{\widetilde{#1}}}}
\ddefloop ABCDEFGHIJKLMNOPQRSTUVWXYZ\ddefloop
\def\ddef#1{\expandafter\def\csname tc#1\endcsname{\ensuremath{\widetilde{\mathcal{#1}}}}}
\ddefloop ABCDEFGHIJKLMNOPQRSTUVWXYZ\ddefloop

\newcommand{\ls}{\ell}
\newcommand{\pmo}{\crl*{\pm{}1}}

\newcommand{\vphi}{\varphi}

\newcommand{\grad}{\nabla}

\newcommand{\ldef}{\vcentcolon=}

\icml{\input{packages}}
\input{widebar}

\newtoggle{editing}
\toggletrue{editing}

\iftoggle{editing}
{
\usepackage[textsize=tiny]{todonotes}
\newcommand{\tododf}[2][]{\todo[size=\scriptsize,color=green!20!white,#1]{DF: #2}}
\newcommand{\todobb}[2][]{\todo[size=\scriptsize,color=blue!20!white,#1]{BB:
    #2}}
}
{
\newcommand{\tododf}[2][]{}
\newcommand{\todobb}[2][]{}
}
\iftoggle{editing}
{
  \usepackage{color-edits}
}
{
  \usepackage[suppress]{color-edits}
}
\addauthor{df}{green!50!black}
\addauthor{bb}{blue!50!black}

\arxiv{
  \usepackage[colorlinks=true, linkcolor=blue!70!black, citecolor=blue!70!black,urlcolor=black,breaklinks=true]{hyperref}
\usepackage[capitalize]{cleveref}
\newcommand{\pref}[1]{\cref{#1}}

}
\icml{
\newcommand{\BlackBox}{\rule{1.5ex}{1.5ex}}  %
\renewcommand{\qedsymbol}{\BlackBox}
\renewenvironment{proof}{\par\noindent{\bf Proof. }}{\hfill\BlackBox\\[2mm]}
}
\arxiv{
\newcommand{\qedarxiv}{\hfill\qedsymbol\\[2mm]}
\renewenvironment{proof}{\par\noindent{\bf Proof. }}{\qedarxiv}
}

\makeatletter
\newcommand{\pushright}[1]{\ifmeasuring@#1\else\omit\hfill$\displaystyle#1$\fi\ignorespaces}
\newcommand{\pushleft}[1]{\ifmeasuring@#1\else\omit$\displaystyle#1$\hfill\fi\ignorespaces}
\makeatother

\newtheorem{example}{Example} 
\newtheorem{theorem}{Theorem}
\newtheorem{lemma}[theorem]{Lemma} 
\newtheorem{proposition}[theorem]{Proposition} 

\newtheorem{corollary}[theorem]{Corollary}
\newtheorem{definition}[theorem]{Definition}

\newcommand\blfootnote[1]{%
  \begingroup
  \renewcommand\thefootnote{}\footnote{#1}%
  \addtocounter{footnote}{-1}%
  \endgroup
}

\makeatletter
\DeclareFontFamily{OMX}{MnSymbolE}{}
\DeclareSymbolFont{MnLargeSymbols}{OMX}{MnSymbolE}{m}{n}
\SetSymbolFont{MnLargeSymbols}{bold}{OMX}{MnSymbolE}{b}{n}
\DeclareFontShape{OMX}{MnSymbolE}{m}{n}{
    <-6>  MnSymbolE5
   <6-7>  MnSymbolE6
   <7-8>  MnSymbolE7
   <8-9>  MnSymbolE8
   <9-10> MnSymbolE9
  <10-12> MnSymbolE10
  <12->   MnSymbolE12
}{}
\DeclareFontShape{OMX}{MnSymbolE}{b}{n}{
    <-6>  MnSymbolE-Bold5
   <6-7>  MnSymbolE-Bold6
   <7-8>  MnSymbolE-Bold7
   <8-9>  MnSymbolE-Bold8
   <9-10> MnSymbolE-Bold9
  <10-12> MnSymbolE-Bold10
  <12->   MnSymbolE-Bold12
}{}

\let\llangle\@undefined
\let\rrangle\@undefined
\DeclareMathDelimiter{\llangle}{\mathopen}%
                     {MnLargeSymbols}{'164}{MnLargeSymbols}{'164}
\DeclareMathDelimiter{\rrangle}{\mathclose}%
                     {MnLargeSymbols}{'171}{MnLargeSymbols}{'171}
\newcommand{\vast}{\bBigg@{4}}
\newcommand{\Vast}{\bBigg@{5}}

\newcommand{\Nats}{\mathbb{N}}
\newcommand{\Nor}{\mathcal{N}}

\newcommand{\Reals}{\mathbb{R}}

\newcommand{\Ord}{\mathcal{O}}
\newcommand{\bigoh}{\cO}
\newcommand{\bigoht}{\tilde{\cO}}

\newcommand{\A}{\mathcal{A}}
\newcommand{\F}{\mathcal{F}}

\newcommand{\1}{\mathds{1}}

\newcommand{\eps}{\varepsilon}
\newcommand{\scfunc}{\varphi}
\newcommand{\expfunc}{\psi}
\newcommand{\dist}{\mathcal{D}}
\newcommand{\borel}{\mathcal{P}}
\newcommand{\borelF}{{\borel}_{\! \scriptscriptstyle \F}}
\newcommand{\hcube}{\mathcal{V}}
\newcommand{\lipF}{\F_\hcube}

\DeclareMathOperator{\interior}{interior}

\newcommand{\vect}[1]{\bm{#1}}
\newcommand{\vectm}[1]{\bm{#1}}

\newcommand{\node}[1]{{#1}_t(\vect{y})}

\DeclareMathOperator*{\newlim}{\mathrm{lim}\vphantom{\mathrm{infsup}}}
\DeclareMathOperator*{\newmin}{\mathrm{min}\vphantom{\mathrm{infsup}}}
\DeclareMathOperator*{\newmax}{\mathrm{max}\vphantom{\mathrm{infsup}}}
\DeclareMathOperator*{\newinf}{\mathrm{inf}\vphantom{\mathrm{infsup}}}
\DeclareMathOperator*{\newsup}{\mathrm{sup}\vphantom{\mathrm{infsup}}}
\renewcommand{\lim}{\newlim}
\renewcommand{\min}{\newmin}
\renewcommand{\max}{\newmax}
\renewcommand{\inf}{\newinf}
\renewcommand{\sup}{\newsup}

\DeclareMathOperator{\EE}{\mathbb{E}}
\DeclareMathOperator{\PP}{\mathbb{P}}

\newcommand{\EEy}{\mathop{\bbE}_{\scriptstyle \vect{y}\sim\vect{p}}}
\newcommand{\EEyn}[1]{\mathop{\bbE}_{\scriptstyle
    \vect{y}_{#1}\sim\vect{p}}}

\newcommand{\KL}[2]{\textnormal{KL}\left(#1 \ \Vert \ #2\right)}

\newcommand{\Fx}{\F \circ \vect{x}}
\newcommand{\Fpx}{\F_{\scriptstyle \vect{p},\vect{x}}}
\newcommand{\Vpx}{V_{\scriptstyle \vect{p},\vect{x}}}

\newcommand{\optf}{f_{\scriptscriptstyle \dist}^*}

\newcommand{\Norm}[1]{\left\lVert#1\right\rVert}
\newcommand{\upp}[1]{^{\scriptscriptstyle(#1)}}

\newcommand{\sn}{\sum_{t=1}^n}

\newcommand{\ip}[2]{\left\langle #1, #2 \right\rangle}
\newcommand{\logloss}{\ell}

\newcommand{\context}{\mathcal{X}}
\newcommand{\pred}{\hat p}
\newcommand{\game}[2]{\left\llangle #1 \right\rrangle_{t=1}^{#2}}
\newcommand{\RegretR}{\mathcal{R}}
\newcommand{\regret}[3]{{\RegretR_{n\!}}\left(#2; #1, #3\right)} %
\newcommand{\logregret}[3]{{\RegretR_{n\!}}\left(#2; #1, #3\right)} %
\newcommand{\minimax}[1]{\RegretR_{n}(#1)} %
\newcommand{\logminimax}[1]{\RegretR_{n}(#1)} %
\newcommand{\RegretU}{\mathrm{U}}
\newcommand{\newbound}[1]{\RegretU_n^{\scriptstyle \mathrm{new}}(#1)} %
\newcommand{\oldbound}[1]{\RegretU_n^{\scriptstyle \mathrm{old}}(#1)} %
\newcommand{\polylog}{\mathrm{polylog}}

\newcommand{\phat}{\hat{p}}

\newcommand{\ellinf}{\ls_{\infty}}
\newcommand{\Linf}{L_{\infty}}
\newcommand{\supnorm}{sup-norm\xspace}

\newcommand{\seqcover}[3]{\Nor_{#3 \!}\left(#1, #2\right)}
\newcommand{\unifcover}[3]{{\widebar\Nor}_{#3}\left(#1, #2\right)}

\newcommand{\seqentropy}[4]{\mathcal{H}_{#4 \!}\left(#1, #2, #3\right)}
\newcommand{\empentropy}[4]{{\hat {\mathcal{H}}}_{#4}\left(#1, #2, #3\right)}

\newcommand{\dsup}{d_{\mathrm{sup}}}

\newcommand{\berndist}{\mathrm{Ber}}
\newcommand{\unifdist}{\mathrm{Unif}}

\renewcommand{\hat}[1]{\wh{#1}}

\newcommand{\trn}{\top}

\icml{
\icmltitlerunning{Tight Bounds on Minimax Regret under Logarithmic Loss via Self-Concordance}
}

\arxiv{
\title{Tight Bounds on Minimax Regret under\\ Logarithmic Loss via Self-Concordance}
  \date{}
\author{
Blair Bilodeau\footnote{University of Toronto, Vector Institute, and Institute for Advanced Study}
\\ {\small \texttt{blair@utstat.toronto.edu}}
\and 
Dylan J. Foster\footnote{Massachusetts Institute of Technology}
\\ {\small\texttt{dylanf@mit.edu}}
\and
Daniel M. Roy\footnotemark[1]
\\ {\small \texttt{droy@utstat.toronto.edu}}
}
}

\begin{document}
\arxiv{\maketitle}
\arxiv{\blfootnote{Published in ICML 2020}}

\icml{
\twocolumn[
\icmltitle{Tight Bounds on Minimax Regret under\\ Logarithmic Loss via Self-Concordance}
\icmlsetsymbol{equal}{*}

\begin{icmlauthorlist}
\icmlauthor{Blair Bilodeau}{toronto,vector,ias}
\icmlauthor{Dylan J. Foster}{mit}
\icmlauthor{Daniel M. Roy}{toronto,vector,ias}
\end{icmlauthorlist}

\icmlaffiliation{toronto}{Statistical Sciences, University of Toronto}
\icmlaffiliation{vector}{Vector Institute}
\icmlaffiliation{mit}{Massachusetts Institute of Technology}
\icmlaffiliation{ias}{Institute for Advanced Study}

\icmlcorrespondingauthor{Blair Bilodeau}{blair.bilodeau@mail.utoronto.ca}
\icmlkeywords{online learning; logarithmic loss; minimax regret; sequential entropy}

\vskip 0.3in
]

\printAffiliationsAndNotice{}  %
}

\begin{abstract}
We consider the classical problem of sequential probability assignment
under logarithmic loss while competing against an arbitrary,
potentially nonparametric class of experts. We obtain tight bounds
on the minimax regret via a new approach that exploits the
self-concordance property of the logarithmic loss. We show that for any
expert class with (sequential) metric entropy $\bigoh(\gamma^{-p})$ at scale $\gamma$, the minimax regret is
$\bigoh(n^{\frac{p}{p+1}})$, and that this rate cannot be improved without additional
assumptions on the expert class under consideration. As an application of our techniques, we resolve the minimax regret for nonparametric Lipschitz classes
of experts.

\end{abstract}

\icml{
  \global\icmlrulercount 0\relax
  }

\section{Introduction}

Sequential probability assignment is a classical problem that has been studied intensely throughout
domains including portfolio optimization
\citep{cover91universal,cover96universal,cross03universal}, information theory
\citep{rissanen84coding,merhav98prediction,xie00minimax}, and---more recently---adversarial
machine learning \citep{goodfellow14gans,grnarova17onlinegans,liang18adversarial}. The
 goal is for a \emph{player} to assign probabilities to an arbitrary, potentially
 adversarially generated sequence of outcomes, and to do so nearly as well as a
 benchmark class of \emph{experts}. More formally, consider
 the following protocol: for rounds $t =
1,\dots,n$, the player receives a \emph{context} $x_t\in\cX$, predicts a probability
$\pred_t \in [0,1]$ (using only the context $x_t$), observes a binary
outcome $y_t \in
\{0,1\}$, and incurs the \emph{logarithmic loss} (``log loss''), defined by 
\begin{align*}
  \logloss(\pred_t, y_t) = -y_t \log(\pred_t) - (1-y_t)\log(1-\pred_t). 
\end{align*}
The log loss penalizes the player based on how much probability mass they place on the
actual outcome. Without distributional assumptions, one cannot control the
total incurred loss, and so it is standard to study the \emph{regret};
that is, the difference between the player's total loss and the total
loss of the single best predictor in a (potentially uncountable)
reference class of experts.
Writing the vector of player predictions as $\hat{\vect{p}} = (\hat
p_1,\dots,\hat p_n)$, and likewise defining $\vect{x} =
(x_1,\dots,x_n)$ and $\vect{y} =
(y_1,\dots,y_n)$, the player's regret with respect to a class of experts $\F \subseteq [0,1]^\context$ is defined as
\begin{align*}
  \regret{\hat{\vect{p}}}{\F}{\vect{x},\vect{y}} = \sn \logloss(\pred_t, y_t) - \inf_{f \in \F} \sn \logloss(f(x_t), y_t). 
\end{align*}
Compared to similar sequential prediction problems found throughout the
literature on online learning \citep{plg06book,hazan16oco,bandit20book}, the
distinguishing feature of the sequential probability assignment
problem is the log loss, which amounts to evaluating the
log-likelihood of the observed outcome under the player's predicted
distribution. Typical results in online learning assume
the loss function to be convex and smooth or
Lipschitz (e.g., absolute loss or square loss on bounded predictions) or at least bounded (e.g.,
classification loss), while the log loss may have unbounded values and
unbounded gradient. Consequently, beyond simple classes of
experts, naively applying the standard tools of online learning leads to loose guarantees; instead, we exploit
refined properties of the log loss to obtain tight regret bounds for
sequential probability assignment.

\paragraph{Minimax Regret.}
We investigate the fundamental limits for sequential probability
assignment through the lens of minimax analysis. We focus on \emph{minimax regret}, defined by
\begin{align}
\hspace{-.5em}
  \minimax{\F}
  = \sup_{x_1} \inf_{\pred_1} \sup_{y_1} \cdots \sup_{x_n} \inf_{\pred_n} \sup_{y_n} \regret{\hat{\vect{p}}}{\F}{\vect{x},\vect{y}},
  \label{eq:rep_op_gen}
\end{align}
where $x_t\in\cX_t$ (defined formally in \cref{sec:prelim}), $\pred_t \in [0,1]$, and $y_t
\in \{0,1\}$ for all $t\in[n]$. The minimax regret expresses
worst-case performance of
the best player across all adaptively chosen data sequences. For
simple (e.g., parametric) classes of experts, the minimax regret is
well-understood, including exact constants
\citep{rissanen1986complexity,rissanen1996fisher,shtarkov87coding,freund2003predicting}. For
rich classes of experts, however, tight guarantees are not known, and
hence our aim in this paper is to answer:
\emph{How does the complexity of $\cF$ shape the minimax regret?}

A standard object used to control the minimax regret in statistical
learning and sequential prediction is the
\emph{covering number}, which is a measure of the complexity of an
expert class $\F$. The covering number for $\F$ is the size of the smallest subset of $\F$ such that every
element of $\F$ is \emph{close} to an element of the subset, where
close is defined for appropriate notions of scale and distance. Early
covering-based bounds for sequential probability assignment \citep{opper99logloss,cesabianchi99logloss} use coarse notions of
distance, but these bounds become vacuous for
sufficiently rich expert classes.

More recently, \citet{rakhlin15binary} gave sharper guarantees that use a finer notion of cover, referred
to as a \emph{sequential cover} (see \cref{def:seq_cover}), which has previously been shown to
characterize the minimax regret for simpler online learning problems
with Lipschitz losses \citep{rakhlin15learning}. Unfortunately, to deal with the fact that log loss is
non-Lipschitz, this result (and all prior work in this line)
approximates regret by truncating the allowed probabilities away from $0$ and
$1$. This approximation forces the gradient of log loss to be bounded,
but leads to suboptimal bounds for rich expert classes. Hence,
\citet{rakhlin15binary} posed the problem of whether the minimax regret for
sequential probability assignment can be characterized using
only the notion of sequential covering. This question is natural, as
the answer is affirmative for the absolute loss
\citep{rakhlin15learning}, square loss \citep{rakhlin14nonparametric}, and other
common Lipschitz losses such as the hinge.

\subsection{Overview of Results}
Our main result is to show that for experts classes $\cF$ for which
the \emph{sequential entropy} (the log of the sequential
covering number) at scale $\gamma$ grows as $\gamma^{-p}$, we have
\[
  \minimax{\F} \leq\ \bigoh(n^{\frac{p}{p+1}}).
\]
This upper bound recovers the best-known rates for all values of $p$,
and offers strict improvement whenever $p>1$ (i.e., whenever the class
$\cF$ is sufficiently complex). We further show that for certain expert
classes---in particular, nonparametric Lipschitz classes over $\brk*{0,1}^{p}$---this rate
cannot be improved. As a consequence, we resolve the minimax
regret for these classes.

An important implication of our results is that for general
classes $\cF$, the optimal rate for regret cannot be characterized
purely in terms of sequential covering numbers; this follows by
combining our improved upper and lower bounds with an earlier
observation from \citet{rakhlin15binary}.

Our upper bounds are obtained through a new technique that exploits
the curvature of log loss (specifically, the property of
\emph{self-concordance}) to bound the regret. This allows us to handle
the non-Lipschitzness of the log loss directly without invoking the
truncation and approximation arguments that lead to suboptimal regret in previous approaches.

\subsection{Related Work} 
For finite expert classes, it is well-known that the
minimax regret $\logminimax{\F}$ is of order
$\log\abs{\F}$ \citep{vovk98game}. Sharp guarantees are also known for
countable expert classes \citep{banerjee06bayesian} and parametric classes
\citep{rissanen1986complexity,rissanen1996fisher,shtarkov87coding,xie00minimax,freund2003predicting,miyaguchi19complexity};
see also Chapter 9 of \citet{plg06book}.

In this work, we focus on obtaining tight guarantees for rich,
nonparametric classes of experts. Previous work in this
direction has obtained bounds for large expert classes using various
notions of complexity for the class.
\citet{opper99logloss, cesabianchi99logloss} bound the minimax
regret under log loss using covering numbers for the expert class
defined with respect to the \supnorm over the context space;
that is, $\dsup(f,g) = \sup_{x\in\cX}\abs*{f(x)-g(x)}$. Covering with respect to all the elements in the domain is rather
restrictive, and there are many cases for which the sup-norm
covering number is infinite even though the class is learnable, or
where the \supnorm cover has undesirable dependence on the dimension
of the context space.

Building on a line of work which characterizes minimax rates for
Lipschitz losses \citep{rakhlin14nonparametric,rakhlin15learning}, \citet{rakhlin15binary} gave
improved upper bounds for sequential probability assignment based on
\emph{sequential covering numbers}, which require that covering
elements are close only on finite sequences of contexts induced by
binary trees. Sequential covering numbers can be much smaller than \supnorm
covers. For example, infinite
dimensional linear functionals do not admit a finite \supnorm cover, but \citet{rakhlin15binary}
 show via sequential covering that they are learnable at a rate of
 $\tilde\Ord(n^{3/4})$. Moreover, \citet{rakhlin15binary} show that sublinear regret is
possible only for expert classes with bounded sequential covering
numbers.

While the rates obtained by \citet{rakhlin15binary} are nonvacuous for
many expert classes, they have
suboptimal order for even moderately complex classes. Indeed, in order to
handle the unbounded gradient of log loss, \citet{rakhlin15binary} rely on truncation of the probabilities allowed to
be predicted: They restrict the probabilities to
$[\delta,1-\delta]$ for some $0 < \delta \leq 1/2$, and then bound the
true minimax regret by the minimax regret subject to this restricted
probability range, plus an error term of size $\Ord(n\delta)$. This
strategy allows one to treat the log loss
as uniformly bounded and $1/\delta$-Lipschitz, but leads to poor rates
compared to more common Lipschitz/strongly convex losses such as the
square loss. Subsequent work by \citet{foster18logistic} gave improvements to
this approach that exploit the
\emph{mixability} property of the log loss \citep{vovk98game}. While their
results lead to improved rates for classes of ``moderate'' complexity, they face similar suboptimality for
high-complexity expert classes.

\subsection{Organization}
\pref{sec:main} presents our improved minimax upper bound for
general, potentially nonparametric expert classes
(\pref{thm:main}). In \cref{ssec:lipschitz}, we instantiate
this bound for concrete examples of expert classes, and present a
lower bound that is tight for certain expert classes
(\pref{thm:lower}). In \cref{ssec:compare} we give a detailed comparison between our
rates and those of prior work as a function of the sequential entropy.

In \cref{sec:proof}, we prove our upper bound via a new
approach based on self-concordance of the log loss. \cref{sec:lower-proof} proves our lower bound, completing our characterization of the minimax regret.

We conclude the paper with a short discussion in
\cref{sec:discussion}.

\section{Preliminaries}
\label{sec:prelim}
\paragraph{Contexts.}
We allow for time-varying context sets. At time $t$, we take $x_t$ to belong to a set $\cX_t\subseteq\cX$,
whose value may depend on the \emph{history}, defined by
  $h_{1:t-1} = (x_1,y_1,\dots,x_{t-1},y_{t-1})$, but not future
  observations. Formally, we have $\context_t: (\context \times
  \{0,1\})^{t-1} \to 2^\context$, so that $x_t \in
  \context_t(h_{1:t-1})$. An example is the
  observed outcomes up to a given round, given by
  $\context_t(h_{1:t-1}) = \{(y_{1:t-1})\}$, which covers the standard setting of \citet{cesabianchi99logloss}. 
  Another example is time-independent information, for example, 
  $\context_t(h_{1:t-1}) = \{\vect{x} \in \Reals^d: \Norm{\vect{x}}
  \leq 1 \}$, which can be viewed analogously to the covariates in a standard regression task.

\paragraph{Sequential Covers and Metric Entropy.}
Sequential covering numbers are defined using \emph{binary trees}
indexed by sequences of binary observations (``paths''). Formally, for a
set $\cA$, an $\A$-valued binary tree $\vect{a}$ of depth $n$ is a sequence of mappings $a_t: \{0,1\}^{t-1} \to \A$ for $t \in [n]$.

For a sequence (path) $\vectm{\eps} \in \{0,1\}^n$ and a tree
$\vect{a}$ of depth $n$, let $a_t(\vectm{\eps}) \ldef
a_t(\eps_1,\dots,\eps_{t-1})$ for $t \in [n]$. Also, denote the
sequence of values a tree $\vect{a}$ takes on a path $\vectm{\eps}$ by
$\vect{a}(\vectm{\eps}) =
(a_1(\vectm{\eps}),\dots,a_n(\vectm{\eps}))$. For a function $f:\A \to
\Reals$, let $f \circ \vect{a}$ denote the tree taking values
$(f(a_1(\vectm{\eps})),\dots,f(a_n(\vectm{\eps})))$ on the path
$\vectm{\eps}$. We extend this notation for a set of functions $\F$
by defining $\F \circ{} \vect{a} = \{f \circ \vect{a}: f \in
\F\}$. Further, we say an $\context$-valued binary tree $\vect{x}$ is
\textit{consistent} if for all rounds $t \in [n]$ and paths $\vect{y}
\in \{0,1\}^n$, $x_t(\vect{y}) \in \context_t(h_{1:t-1})$. For the
remainder of this paper we will only consider context trees $\vect{x}$
with this property.

The notion of trees allows us to formally define a \emph{sequential
  cover}, which may be thought of as a generalization of the classical
notion of empirical covering that encodes the dependency structure of the online game.

\begin{definition}[\citealt{rakhlin15learning}]\label{def:seq_cover}
  Let $\cA$ and $V$ be collections of $\bbR$-valued binary trees of depth $n$. $V$ is a
  sequential cover for $\A$ at scale $\gamma$ if
\begin{align*}
  \max_{\vect{y} \in \{0,1\}^n} \sup_{\vect{a} \in \A} \inf_{\vect{v} \in V} \max_{t \in [n]} \ \abs{\node{a} - \node{v}} \leq \gamma.
\end{align*}
Let $\seqcover{\A}{\gamma}{\infty}$ be the size of the smallest such
cover.\footnote{The ``$\infty$'' subscript reflects that the
  cover is defined with respect to the empirical $\Linf$ norm.}
\end{definition}

For a function class $\cF$, we define the \emph{sequential entropy} of
$\F$ at scale $\gamma$ and depth $n$ as the log of the worst-case
sequential covering number:
\begin{align*}
  \seqentropy{\F}{\gamma}{n}{\infty} = \sup_{\vect{x}} \log \, \seqcover{\Fx}{\gamma}{\infty},
\end{align*}
where the $\sup$ is taken over all context trees of depth $n$.

Sequential covering numbers incorporate the dependence structure of
online learning, and consequently are never smaller than classical
\emph{empirical covers} found in statistical learning, which require
that the covering elements are close only on a fixed sequence
$x_{1:n}$. While the sequential covering number of $\Fx$ for context trees of
depth $n$ will never be smaller than the empirical covering number for datasets of size $n$, it will---importantly---always be
finite. Additionally, because the definition allows one to choose the covering element as a function of the path
$\vect{y}$, the sequential covering number at depth $n$ is typically
much smaller than, for example, the empirical covering number of size $2^n$, despite the context tree having $2^n - 1$ unique values.

\paragraph{Asymptotic Notation.}
We adopt standard big-oh notation.
Consider two real-valued sequences $(x_n)$ and $(y_n)$. We write $x_n \leq
\Ord(y_n)$ if there exists a constant $M>0$ such that for all sufficiently large $n$,
$\abs{x_n} \leq M y_n$. Conversely, $x_n \geq \Omega(y_n)$ if $y_n
\leq \Ord(x_n)$. We write $x_n \leq \tilde \Ord(y_n)$ if there is some
$r>0$ such that $x_n \leq \Ord(y_n (\log(n))^r)$. We also write $x_n =
\Theta(y_n)$ if $\Omega(y_n) \leq x_n \leq \Ord(y_n)$, and similarly
$x_n = \tilde\Theta(y_n)$ if $\Omega(y_n) \leq x_n \leq
\tilde\Ord(y_n)$. Note that we do not specify a notion of $\tilde \Omega$. Instead, we say a sequence $x_n = \polylog(y_n)$ if there exist some $0<r<s$ such that $\Omega((\log(y_n))^r) \leq x_n \leq \Ord((\log(y_n))^s)$. Then, for any function $g$, $x_n \leq \Ord(g(\polylog(y_n)))$ if there exists some sequence $y'_n = \polylog(y_n)$ such that $x_n \leq \Ord(g(y'_n))$.

\section{Minimax Regret Bounds}
\label{sec:main}

We now state our upper bound on the minimax regret for sequential
probability assignment. Our result is non-constructive; that is, we do
not provide an explicit algorithm that achieves our upper bound. Rather, we
characterize the fundamental limits of the learning
problem for arbitrary expert classes, providing a benchmark for
algorithm design going forward.

\begin{theorem}\label{thm:main}
For any $\context$ and $\F \subseteq [0,1]^\context$,
\begin{align*}
   \logminimax{\F} 
   \leq \inf_{\gamma > 0} \Big\{4n \gamma + c \, \seqentropy{\F}{\gamma}{n}{\infty} \Big\},
\end{align*}
where $c = \frac{2 - \log(2)}{\log(3) - \log(2)}\leq4$.
\end{theorem}
For simple parametric classes where
$\seqentropy{\cF}{\gamma}{n}{\infty}=\Theta(d\log(1/\gamma))$,
\pref{thm:main} recovers the usual fast rate of $\bigoh(d\log{}(en/d))$.
More interesting is the rich/high-complexity regime where
$\seqentropy{\F}{\gamma}{n}{\infty} = \Theta(\gamma^{-p})$ for $p>0$,
for which \cref{thm:main} implies that
\begin{align}\label{eqn:example}
  \logminimax{\F} \leq \Ord(n^{\frac{p}{p+1}}).
\end{align}
As we discuss at length in \pref{ssec:compare}, this rate improves
over prior work for all $p>1$. More importantly, this upper bound is tight for
certain nonparametric classes (namely, the $1$-Lipschitz experts). That is, if one wishes to bound regret only in terms of sequential entropy, \cref{thm:main} cannot be improved.
\begin{theorem}\label{thm:lower}
For any $p \in \Nats$, the class $\F$ of $1$-Lipschitz (w.r.t. $\ellinf$) experts on $[0,1]^p$ satisfies $\seqentropy{\F}{\gamma}{n}{\infty} = \Theta(\gamma^{-p})$ and
\begin{align*}
  \logminimax{\F} = \Theta(n^{\frac{p}{p+1}}).
\end{align*}
\end{theorem}

While \pref{thm:lower} shows that our new upper bound cannot be
improved in a worst-case sense, there is still room for improvement
for specific function classes of interest.
  Let $\bbB_2$ be the unit ball in a Hilbert space.
  Consider the class of infinite-dimensional linear
  predictors $\F = \{ x \mapsto \frac{1}{2}\left[\ip{w}{x} + 1
  \right] \mid{} w \in \bbB_2\}$, with $\cX=\bbB_2$. This class has sequential entropy
  $\seqentropy{\F}{\gamma}{n}{\infty} = \tilde\Theta(\gamma^{-2})$, so
  $\logminimax{\cF}\leq{}\bigoht(n^{2/3})$ by
  \pref{thm:main}. However, \citet{rakhlin15binary}
  describe an explicit algorithm that attains regret
  $\tilde\bigoh(n^{1/2})$ for this class,
  meaning that our upper bound is loose for this example. Yet, since
  \pref{thm:lower} shows that the upper bound cannot be improved
  without further assumptions, we draw the following conclusion.

  \begin{corollary}
    The minimax rates for sequential probability assignment with the
    log loss cannot be characterized purely in terms of sequential entropy.
  \end{corollary}

We discuss a couple more features of \pref{thm:main} and \pref{thm:lower}
below.
\begin{itemize}[itemsep=0pt,topsep=0pt]
  \item The proof strategy for \pref{thm:main} differs from previous
    approaches by discretizing $\cF$ at a single scale rather
    than multiple scale levels (referred to as
    chaining). Surprisingly, this rather coarse approach achieves the
    previous best known results 
    and improves on them for rich expert classes. Key to this improvement
    is the self-concordance of the log loss, which enables us to
    avoid truncation arguments.
\item \cref{thm:lower} in fact lower
  bounds the minimax regret when data is generated i.i.d.\ from a
  well-specified model, which implies that for 
  Lipschitz classes, this apparently easier setting is in fact just as
  hard as the fully adversarial setting. This is in contrast to the case for square loss,
  where the rates for the i.i.d.\ well-specified and
  i.i.d.\ misspecified settings diverge once $p\geq{}2$ \citep{rakhlin17minimax}.
\end{itemize}

\subsection{Further Examples}
\label{ssec:lipschitz}

In order to place our new upper bound in the context of familiar
expert classes, we walk through some additional examples below.

\newcommand{\RadSeq}{\mathfrak{R}_n}
\begin{example}[Sequential Rademacher Complexity]
  \label{ex:rademacher}
  The sequential Rademacher complexity of an expert class $\cF$ is
  given by
  \[
    \RadSeq(\cF) = \sup_{\vect{x}}\En_{\eps}\sup_{f\in\cF}\sum_{t=1}^{n}\eps_t \, f(x_t(\eps)),
  \]
  where $\sup_{\vect{x}}$ ranges over all $\cX$-valued trees and
  $\eps\in\pmo^{n}$ are Rademacher random variables.\footnote{Here we
    overload the definition of a tree in the natural way to allow
    arguments in $\crl{\pm{}1}$ rather than $\crl{0,1}$.} Via Corollary 1 and Lemma 2
   of \citet{rakhlin15martingale}, we deduce that
  \[
    \logminimax{\cF} \leq{} \bigoht\prn*{\RadSeq^{2/3}(\cF)\cdot{}n^{1/3}}.
  \]
\end{example}

\begin{example}[Smooth Nonparametric Classes]
Let $\cF$ be the class of all bounded functions over $\brk*{0,1}^{d}$
for which the first $k-1$ derivatives are Lipschitz. Then we may take $p = d/k$ \citep[see, e.g.,
Example~5.11 of][]{wainwright19book}, and hence \pref{thm:main} gives that
$\logminimax{\F} \leq \bigoht(n^{\frac{d}{d+k}})$. One can show that this
is optimal via a small modification to the proof of \pref{thm:lower}.
\end{example}

\begin{example}[Neural Networks]
  \citet{rakhlin15learning} show that neural networks with Lipschitz
  activations and $\ls_1$-bounded weights have
  $\RadSeq(\cF)\leq \bigoht(\sqrt{n})$. We conclude from
  \pref{ex:rademacher} that $\logminimax{\cF}\leq\bigoht(n^{2/3})$ for
  these classes.
\end{example}

\subsection{Comparing to Previous Regret Bounds}
\label{ssec:compare}

\icml{
We now compare the bound from \cref{thm:main} to the previous state of the art, Theorem~7 of \citet{foster18logistic}, which shows that for any $\context$ and $\F \subseteq [0,1]^\context$,

{\small
\begin{align}
  \logminimax{\F}
  &\leq \inf_{\overset{\gamma \geq \alpha > 0}{\delta > 0}} 
  \bigg\{\frac{4n \alpha}{\delta} + 30 \sqrt{\frac{2n}{\delta}}\int_\alpha^\gamma \! \sqrt{\seqentropy{\F}{\eps}{n}{\infty}} \, \mathrm{d}\eps 
  \nonumber \\[-5pt]
  & \qquad \qquad\quad + \frac{8}{\delta}\int_\alpha^\gamma \seqentropy{\F}{\eps}{n}{\infty} \mathrm{d}\eps \nonumber \\[-5pt]
  & \qquad \qquad\quad + \seqentropy{\F}{\gamma}{n}{\infty} + 3n\delta\log(1/\delta) \bigg\}.
  \label{eq:oldbound}
\end{align}}}

\arxiv{
We now compare the bound from \cref{thm:main} to the previous state of the art, Theorem~7 of \citet{foster18logistic}, which shows that for any $\context$ and $\F \subseteq [0,1]^\context$, $\logminimax{\F}$ is upper bounded by
\begin{align}
  \inf_{\overset{\gamma \geq \alpha > 0}{\delta > 0}} 
  \bigg\{\frac{4n \alpha}{\delta} + 30 \sqrt{\frac{2n}{\delta}}\int_\alpha^\gamma \! \sqrt{\seqentropy{\F}{\eps}{n}{\infty}} \, \mathrm{d}\eps 
  + \frac{8}{\delta}\int_\alpha^\gamma \seqentropy{\F}{\eps}{n}{\infty} \mathrm{d}\eps
  + \seqentropy{\F}{\gamma}{n}{\infty} + 3n\delta\log(1/\delta) \bigg\}.
  \label{eq:oldbound}
\end{align}}

For any expert class $\F$ we will refer to the upper bound of \cref{thm:main} by $\newbound{\F}$ and the upper bound of \citet[Theorem~7]{foster18logistic} by $\oldbound{\F}$. We observe the following relationship, proven in Appendix~\ref{prf:compare}.

\begin{proposition}\label{prop:compare}
For any $\context$ and $\F \subseteq [0,1]^\context$, the following hold:
\begin{enumerate}[topsep=0pt]
\item[\upshape{(i)}] If \,
$\seqentropy{\F}{\gamma}{n}{\infty} = \Theta(\log(1/\gamma))$,
\begin{align*}
  \frac{\newbound{\F}}{\oldbound{\F}}
  = \Theta\left(1 \right).
\end{align*}

\item[\upshape{(ii)}] If \, $\seqentropy{\F}{\gamma}{n}{\infty} = \Theta(\gamma^{-p})$ \, for $p \leq 1$,
\begin{align*}
  \frac{\newbound{\F}}{\oldbound{\F}} = \Theta\left(\frac{1}{\polylog(n)} \right).
\end{align*}

\item[\upshape{(iii)}] If \, $\seqentropy{\F}{\gamma}{n}{\infty} = \Theta(\gamma^{-p})$ \, for $p>1$,
\begin{align*}
  \frac{\newbound{\F}}{\oldbound{\F}} = \Theta\left(\frac{1}{n^{\frac{p-1}{2p(p+1)}} \polylog(n)} \right).
\end{align*}
\end{enumerate}
\end{proposition}
\icml{\section{Proof of \cref{thm:main}}}
\arxiv{\section{Proof of Theorem~\ref{thm:main}}}
\label{sec:proof}

We now prove our main result. The proof has three parts. First, we
use a minimax theorem to move to the dual of the online learning game,
where we can evaluate the optimal strategy for the learner. This
allows us to express the value of the minimax regret as a dependent
empirical processes. For the next step, we move to a simpler,
linearized upper bound on this process
using the self-concordance property of the log loss, leading to a
particular ``offset'' process. For finite classes, we can
directly bound the value of the offset process by $\log\abs*{\cF}$;
the final bound in \pref{thm:main} follows by applying this result
along with a discretization argument.

Before proceeding, we elaborate on the second point above. Let
us take a step back and consider the simpler problem of bounding the
minimax regret for square loss. \citet{rakhlin14nonparametric} show that
via a similar minimax theorem, it is possible to bound
the regret by a dependent random process called the \emph{offset sequential
  Rademacher complexity}, which, informally, takes the form
\begin{equation}
  \En\sup_{f\in\cF}\brk*{X_{\textrm{emp}}(f) -
    Y_{\textrm{offset}}(f)}.\label{eq:offset}
\end{equation}
Here, $X_{\mathrm{emp}}(f)$ is a zero-mean Rademacher process indexed
by $\cF$ and $Y_{\textrm{offset}}(f)$ is a quadratic offset. The
offset component arises due to the strong convexity of the square
loss, and penalizes large fluctuations in the Rademacher process,
leading to fast rates.

For the log loss, the issue faced if
one attempts to apply the same strong convexity-based argument is that
the process $X_{\textrm{emp}}(f)$, which involves the derivative of
the loss, becomes unbounded as $f$ approaches the boundary of
$\brk*{0,1}$, and the quadratic offset $Y_{\textrm{offset}}(f)$ does not
  grow fast enough to neutralize it. The simplest way to address this
  issue, and the one taken by \citet{rakhlin15binary}, is to truncate
  predictions. Our main insight is that using self-concordance of the
  log loss rather than strong convexity leads to an offset that can neutralize the derivative, removing the need for truncation and resulting in faster rates. The inspiration for using this
  property came from \citet[Section~6]{rakhlin15binary}, who design a
  variant of mirror descent using a self-concordant barrier as a regularizer
  to obtain fast rates for linear prediction with the log loss, though
  our use of the property here is technically quite different.

  %

%
%

%
%
%

%
%

%
%
%

%

%
%
%
\subsection{Minimax Theorem and Dual Game}
As our first step, we move to the \emph{dual game} in
which the order of max and min at each time step is swapped.
Moving to the dual game is a now-standard strategy
\citep{abernethy09duality,rakhlin14nonparametric,rakhlin15learning,rakhlin15binary,foster18logistic}, and is a useful tool
for analysis because the optimal strategy for the learner is much more
tractable to compute in the dual. 

In particular, for our sequential probability assignment setting, the
following minimax theorem (\pref{prf:dual}) holds.
\begin{lemma}\label{lem:dual}
For any $\context$ and $\F \subseteq [0,1]^\context$,
\icml{
\begin{align*}
  \minimax{\F}
  &=\sup_{x_1} \sup_{p_1 \in [0,1]} \mathop{\EE}_{y_1 \sim p_1} \cdots \sup_{x_n} \sup_{p_n \in [0,1]} \mathop{\EE}_{y_n \sim p_n} \sup_{f \in \F} \nonumber \\ 
  &\qquad \sn \left\{\inf_{\pred_t \in [0,1]} \mathop{\EE}_{y_t \sim p_t} [\logloss(\pred_t, y_t)] - \logloss(f(x_t), y_t) \right\}.
\end{align*}}
\arxiv{
\begin{align*}
  \minimax{\F}
  &=\sup_{x_1} \sup_{p_1 \in [0,1]} \mathop{\EE}_{y_1 \sim p_1} \cdots \sup_{x_n} \sup_{p_n \in [0,1]} \mathop{\EE}_{y_n \sim p_n} \sup_{f \in \F}
  \sn \left\{\inf_{\pred_t \in [0,1]} \mathop{\EE}_{y_t \sim p_t} [\logloss(\pred_t, y_t)] - \logloss(f(x_t), y_t) \right\}.
\end{align*}}
\end{lemma}

The parameter $p_t\in\brk*{0,1}$ represents a distribution over the
adversary's outcome $y_t\in\crl*{0,1}$, which the player can observe before they select $\hat{p}_t$.
For log loss, it is easy to see that the infimum of the interior expectation in 
\cref{lem:dual}
is achieved at $\pred_t = p_t$, so by the linearity of expectation the minimax regret can be written as
\begin{align*}
  \sup_{x_1} \sup_{p_1 \in [0,1]} \mathop{\EE}_{y_1 \sim p_1} \cdots \sup_{x_n} \sup_{p_n \in [0,1]} \mathop{\EE}_{y_n \sim p_n} \logregret{\vect{p}}{\F}{\vect{x},\vect{y}}.
\end{align*}

We simplify this statement using the tree notation from \cref{sec:prelim}. In particular, writing $\EEy$ to denote the nested conditional expectations
$\mathop{\EE}_{y_t \sim p_t(\vect{y})}$ 
for each $t \in [n]$, we can write the minimax regret as
\begin{align*}
  \logminimax{\F} = \sup_{\vect{x}, \vect{p}} \EEy \logregret{\vect{p}(\vect{y})}{\F}{\vect{x}(\vect{y}),\vect{y}},
\end{align*}
where $\vect{x}$ and $\vect{p}$ are respectively $\context$- and $[0,1]$-valued binary trees of depth $n$. 
We now fix an arbitrary context tree $\vect{x}$ and probability tree $\vect{p}$, and show that the bound of \cref{thm:main} holds for $\EEy \logregret{\vect{p}(\vect{y})}{\F}{\vect{x}(\vect{y}),\vect{y}}$.
Recall that there is a $\mathop{\sup}_{f \in \F}$ inside
$\logregret{\vect{p}(\vect{y})}{\F}{\vect{x}(\vect{y}),\vect{y}}$, 
so we must control the expected supremum of a dependent empirical process.

\subsection{Self-Concordance and Offset Process}
As sketched earlier, the key step in our proof is to upper bound 
$\logregret{\vect{p}(\vect{y})}{\F}{\vect{x}(\vect{y}),\vect{y}}$
in terms of a new type of offset
process using self-concordance. Let us first introduce the property formally.
\begin{definition} \label{def:sc}
A function $F: \Reals^d \to \Reals$ is self-concordant on $S \subseteq \Reals^d$ if for all $s \in \interior(S)$ and $h \in \Reals^d$,
\begin{align*}
  \frac{d}{d\alpha} \grad^2 F(s + \alpha h) \Big\rvert_{\alpha=0} \preccurlyeq 2 \grad^2 F(s) \sqrt{h^{\trn} \grad^2 F(s) h}.
\end{align*}
If $F: \Reals \to \Reals$, this can be written as
\begin{align*}
  \abs{F'''(s)} \leq 2 F''(s)^{3/2}.
\end{align*}
\end{definition}
The class of self-concordant functions was first introduced by
\citet{nesterov94book} to study interior point methods. The logarithm is in fact the defining self-concordant function, satisfying equality in \cref{def:sc}. Consequently, we are able to apply the following result about self-concordance to log loss (viewed as a function of the predictions). 
\begin{lemma}[\citealt{nesterov2004introductory}, Theorem~4.1.7]\label{lem:sc_og}
If $F: S\to\bbR$ is self-concordant on a convex set $S$, then for all $s,t \in \interior(S)$,
\begin{align*}
  F(t) \geq F(s) + \ip{\grad F(s)}{t - s} + w\left(\Norm{t - s}_{F,s}\right),
\end{align*}
where $w(z) = z - \log(1+z)$ and $\Norm{h}_{F,s} = \sqrt{h^{\trn} \grad^2 F(s) h}$ is the local norm with respect to $F$.
\end{lemma}

We use \cref{lem:sc_og} to linearize the log loss, leading to a
decomposition similar to (\ref{eq:offset}); note that we only require the scalar version of the lemma. This decomposition allows
us to exploit the fact that, while both the logarithm's value and its derivative tend to
infinity near the boundary, the value does so at a much slower
rate.
\begin{lemma}\label{lem:sc}
\icml{
Let $\eta(p, y) = \frac{d}{dp} \logloss(p, y)$ for $p \in [0,1]$ and $y \in \{0,1\}$, and define $\scfunc(z) = z - \abs{z} + \log(1+\abs{z})$. Then,
$\logregret{\vect{p}(\vect{y})}{\F}{\vect{x}(\vect{y}),\vect{y}}$
is bounded above almost surely by
\begin{align*}
  \sup_{f \in \F} \sn \scfunc\Big(\eta(\node{p}, y_t) [\node{p} - f(\node{x})]\Big)
\end{align*}
under $\vect{y} \sim \vect{p}$.}
\arxiv{
Let $\eta(p, y) = \frac{d}{dp} \logloss(p, y)$ for $p \in [0,1]$ and $y \in \{0,1\}$, and define $\scfunc(z) = z - \abs{z} + \log(1+\abs{z})$. Then, almost surely under $\vect{y} \sim \vect{p}$,
\begin{align*}
  \logregret{\vect{p}(\vect{y})}{\F}{\vect{x}(\vect{y}),\vect{y}} \leq 
  \sup_{f \in \F} \sn \scfunc\Big(\eta(\node{p}, y_t) [\node{p} - f(\node{x})]\Big).
\end{align*}}
\end{lemma}
In the language of (\ref{eq:offset}), we can interpret the linear term
$z$ in $\vphi(z)=z-(\abs*{z}-\log(1+\abs*{z}))$ as giving rise to a
mean-zero process, while the term
$-(\abs{z}-\log(1+\abs*{z}))$ is a (negative) offset that behaves like a
quadratic for small values of $z$ and like the absolute value for large values.

\begin{proof}
Taking derivatives of $\logloss(p, y)$ with respect to $p$, 
\icml{
\begin{align*}
  &\logloss'(p,y) = \frac{-y}{p} + \frac{1-y}{1-p}, \\
  &\logloss''(p,y) = \frac{y}{p^2} + \frac{1-y}{(1-p)^2}, \quad\text{ and } \\
  &\logloss'''(p,y) = \frac{-2y}{p^3} + \frac{2(1-y)}{(1-p)^3}.
\end{align*}}
\arxiv{
\begin{align*}
  \logloss'(p,y) = \frac{-y}{p} + \frac{1-y}{1-p}, \
  \logloss''(p,y) = \frac{y}{p^2} + \frac{1-y}{(1-p)^2}, \text{ and } \,
  \logloss'''(p,y) = \frac{-2y}{p^3} + \frac{2(1-y)}{(1-p)^3}.
\end{align*}}
Since $y \in \{0,1\}$, $\abs{\logloss'''(p,y)} = 2
\logloss''(p,y)^{3/2}$, so the log loss is indeed self-concordant in $p$ on $(0,1)$. Now, fix $y \in \{0,1\}$ and $t \in [n]$, and consider $F(a) = \logloss(a,y)$. We can then apply \cref{lem:sc_og} to $F$ evaluated at $p = \node{p} \in (0,1)$ and $f = f(\node{x}) \in (0,1)$. This gives
\begin{align}
  F(p) - F(f)
  &\leq (p-f)F'(p) - w(\Norm{p - f}_{F, p}).\label{eq:sc_inequality}
\end{align}
By definition, $(p-f)F'(p) = (p-f)\eta(p,y)$. Further,
\begin{align*}
  \Norm{p - f}_{F, p}
  = \sqrt{(p-f)^2F''(p)}.
\end{align*}
Finally, since $y \in \{0,1\}$, $\logloss''(p, y) = \eta(p, y)^2$, so
\begin{align*}
  \Norm{p - f}_{F, p}
  = \abs{(p-f)\eta(p,y)}.
\end{align*}
Applying the definition of $w(z)$ gives the result on $(0,1)$. For the
boundary points $p \in \{0,1\}$ and $f \in \{0,1\}$, it is easy to
check the inequality holds by observing that $p=0$ implies $y=0$ a.s.\
and $p=1$ implies $y=1$ a.s.; we complete this calculation in \pref{lem:sc_edge_case}.
\end{proof}

\subsection{Applying Sequential Covering}
We now follow the standard
strategy of covering the expert class $\cF$, bounding the supremum for
the cover, and then paying a penalty for approximation.

Consider the class of trees $\Fpx = \{\vect{p} - (f\circ\vect{x}): f
\in \F\}$. Our goal is to obtain a bound in terms of the sequential
entropy of this class, which we observe is the same as the sequential
entropy of $\Fx$. Fix some $\gamma > 0$, and let $\Vpx$ be a
sequential cover of $\Fpx$ at scale $\gamma$. 
\icml{Then, by adding and subtracting terms after applying \cref{lem:sc}, 
$\EEy \logregret{\vect{p}(\vect{y})}{\F}{\vect{x}(\vect{y}),\vect{y}}$ is bounded above by
\begin{align}
  &\hspace{-0.8em}\EEy \sup_{\vect{g} \in \Fpx} \min_{\vect{v} \in \Vpx} \sn \nonumber \\
  &\hspace{0.5em} \bigg\{\scfunc\Big(\eta(\node{p}, y_t)\node{g}\Big) 
  - \scfunc\Big(\eta(\node{p}, y_t)\node{v}\Big)\bigg\} \label{eq:approximation} \\ 
  &\hspace{0.5em} + \EEy \max_{\vect{v} \in \Vpx} \sn \scfunc\Big(\eta(\node{p}, y_t)\node{v}\Big) \label{eq:estimation}.
\end{align}}
\arxiv{Then, by adding and subtracting terms after applying \cref{lem:sc}, 
\begin{align}
  \EEy \logregret{\vect{p}(\vect{y})}{\F}{\vect{x}(\vect{y}),\vect{y}}
  &\leq \EEy \sup_{\vect{g} \in \Fpx} \min_{\vect{v} \in \Vpx} \sn
  \bigg\{\scfunc\Big(\eta(\node{p}, y_t)\node{g}\Big) 
  - \scfunc\Big(\eta(\node{p}, y_t)\node{v}\Big)\bigg\} \label{eq:approximation} \\ 
  &\hspace{2em} + \EEy \max_{\vect{v} \in \Vpx} \sn \scfunc\Big(\eta(\node{p}, y_t)\node{v}\Big) \label{eq:estimation}.
\end{align}}
We have now reduced the problem to controlling the
approximation error (\ref{eq:approximation}) and the finite class process
(\ref{eq:estimation}). Controlling the approximation error is handled
by the following property of the function $\vphi$, which we prove in \cref{prf:approximation}.
\begin{lemma}\label{lem:approximation}
For any $s,t \in \Reals$, $\scfunc(s) - \scfunc(t) \leq 2\abs{s-t}$.
\end{lemma}
Applying \cref{lem:approximation}, the approximation error
term (\ref{eq:approximation}) is bounded above by
\icml{
\begin{align}
  &\hspace{-1em}2 \EEy \sup_{\vect{g} \in \Fpx} \min_{\vect{v} \in \Vpx} \sn \Big\lvert\eta(\node{p},y_t)[\node{g} - \node{v}]\Big\rvert \nonumber \\
  &\leq 2 \gamma \EEy \sn \Big\lvert\eta(\node{p}, y_t)\Big\rvert,
  \label{eq:approximation_exp}
\end{align}}
\arxiv{
\begin{align}
  2 \EEy \sup_{\vect{g} \in \Fpx} \min_{\vect{v} \in \Vpx} \sn \Big\lvert\eta(\node{p},y_t)[\node{g} - \node{v}]\Big\rvert
  \leq 2 \gamma \EEy \sn \Big\lvert\eta(\node{p}, y_t)\Big\rvert,
  \label{eq:approximation_exp}
\end{align}}
where we have used the fact that $\Vpx$ is a sequential cover of $\Fpx$ at scale $\gamma$.

For any particular realization of $\vect{y}$, the value of $\eta(\node{p},y_t)$ in
(\ref{eq:approximation_exp}) depends inversely on $p_t$ and
$1-p_t$, and so can be arbitrarily large. Luckily, we recognize that the
large values of $\eta$ are exactly controlled by the small probability
of paths that generate them. That is, adopting the shorthand
$\vect{y}_t = y_{1:t}$,
\icml{
\begin{align*}
  &\hspace{-1em}\EEy \sn \Big\lvert\eta(\node{p}, y_t)\Big\rvert \nonumber \\
  &= \EEyn{n-1} \mathop{\EE}_{y_n \sim p_n(\vect{y})} \left[\sn \left(\frac{y_t}{\node{p}} + \frac{1-y_t}{1-\node{p}} \right) \right] \nonumber \\
  &= \EEyn{n-1} \Bigg[\sum_{t=1}^{n-1} \left(\frac{y_t}{\node{p}} + \frac{1-y_t}{1-\node{p}} \right) \nonumber \\
  &\qquad \qquad + \mathop{\EE}_{y_n \sim p_n(\vect{y})} \left[\left(\frac{y_n}{p_n(\vect{y})} + \frac{1-y_n}{1-p_n(\vect{y})} \right) \right]  \Bigg] \nonumber \\
  &= \EEyn{n-1} \sum_{t=1}^{n-1} \Big\lvert\eta(\node{p}, y_t)\Big\rvert + 2.
\end{align*}}
\arxiv{
\begin{align*}
  \EEy \sn \Big\lvert\eta(\node{p}, y_t)\Big\rvert
  = \ & \EEyn{n-1} \mathop{\EE}_{y_n \sim p_n(\vect{y})} \left[\sn \left(\frac{y_t}{\node{p}} + \frac{1-y_t}{1-\node{p}} \right) \right] \nonumber \\
  = \ & \EEyn{n-1} \Bigg[\sum_{t=1}^{n-1} \left(\frac{y_t}{\node{p}} + \frac{1-y_t}{1-\node{p}} \right)
  + \mathop{\EE}_{y_n \sim p_n(\vect{y})} \left[\left(\frac{y_n}{p_n(\vect{y})} + \frac{1-y_n}{1-p_n(\vect{y})} \right) \right]  \Bigg] \nonumber \\
  = \ & \EEyn{n-1} \sum_{t=1}^{n-1} \Big\lvert\eta(\node{p}, y_t)\Big\rvert + 2.
\end{align*}}

Iterating this argument down to $t=1$ gives 
\begin{align}
  \EEy \sn \Big\lvert\eta(\node{p}, y_t)\Big\rvert = 2n.
  \label{eq:eta_exp}
\end{align}
It remains to control the value of the finite-class process in
(\ref{eq:estimation}). For this we use the offset property, and again
exploit the fact that the $\eta$ term only takes large values on paths
with low probability.

For a $[0,1]$-valued tree $\vect{p}$, we say that a $[-1,1]$-valued tree $\vect{v}$ is a $[\vect{p}-1,\vect{p}]$-valued tree if for all $t \in [n]$ and $\vect{y} \in \{0,1\}^n$, $\node{v} \in [\node{p}-1,\node{p}]$. We have the following bound.

\begin{lemma}\label{lem:estimation}
Consider a $[0,1]$-valued binary tree $\vect{p}$ and a finite class $V$ of $[\vect{p}-1, \vect{p}]$-valued trees. Then
\icml{
\begin{align*}
  &\EEy \max_{\vect{v} \in V} \sn \scfunc\Big(\eta(\node{p},y_t) \node{v} \Big)  
  \leq c \log\abs{V},
\end{align*}
where $c = \frac{2 - \log(2)}{\log(3) - \log(2)}$.}
\arxiv{
\begin{align*}
  &\EEy \max_{\vect{v} \in V} \sn \scfunc\Big(\eta(\node{p},y_t) \node{v} \Big)  
  \leq \frac{2 - \log(2)}{\log(3) - \log(2)} \log\abs{V}.
\end{align*}}
\end{lemma}
\begin{proof}
First, for all $\lambda > 0$, we have
\icml{
\begin{align*}
  &\EEy \max_{\vect{v} \in V} \sn \scfunc\Big(\eta(\node{p},y_t) \node{v} \Big) \\
  &= \log\left(\exp\bigg\{\lambda \frac{1}{\lambda} \EEy \max_{\vect{v} \in V} \sn \scfunc\Big(\eta(\node{p},y_t) \node{v} \Big) \bigg\} \right) \\
  &\leq \frac{1}{\lambda} \log\left(\EEy \exp\bigg\{\lambda \max_{\vect{v} \in V} \sn \scfunc\Big(\eta(\node{p},y_t) \node{v} \Big) \bigg\} \right) \\
  &\leq \frac{1}{\lambda} \log\left(\sum_{\vect{v} \in V} \EEy \exp\bigg\{\lambda \sn \scfunc\Big(\eta(\node{p},y_t) \node{v} \Big) \bigg\} \right), 
\end{align*}}
\arxiv{
\begin{align*}
  \EEy \max_{\vect{v} \in V} \sn \scfunc\Big(\eta(\node{p},y_t) \node{v} \Big)
  &= \log\left(\exp\bigg\{\lambda \frac{1}{\lambda} \EEy \max_{\vect{v} \in V} \sn \scfunc\Big(\eta(\node{p},y_t) \node{v} \Big) \bigg\} \right) \\
  &\leq \frac{1}{\lambda} \log\left(\EEy \exp\bigg\{\lambda \max_{\vect{v} \in V} \sn \scfunc\Big(\eta(\node{p},y_t) \node{v} \Big) \bigg\} \right) \\
  &\leq \frac{1}{\lambda} \log\left(\sum_{\vect{v} \in V} \EEy \exp\bigg\{\lambda \sn \scfunc\Big(\eta(\node{p},y_t) \node{v} \Big) \bigg\} \right), 
\end{align*}}
where the first inequality is Jensen's and the second follows because
the maximum is contained in the sum. Now, for any fixed tree $\vect{v}$,
\icml{
\begin{align}
  &\EEy \exp\bigg\{\lambda \sn \scfunc\Big(\eta(\node{p},y_t) \node{v} \Big) \bigg\} \nonumber \\
  &= \EEyn{n-1} \mathop{\mathbb{E}}_{y_n \sim p_n(\vect{y})} \bigg[ \exp\bigg\{\lambda \sn \scfunc\Big(\eta(\node{p},y_t) \node{v} \Big) \bigg\} \bigg] \nonumber \\
  &= \EEyn{n-1} \bigg[ \exp\bigg\{\lambda \sum_{t=1}^{n-1} \scfunc\Big(\eta(\node{p},y_t) \node{v} \Big) \bigg\} 
  \times \nonumber \\ 
  &\qquad \qquad\qquad
  \expfunc_{p_n(\vect{y}), \lambda}\Big(v_n(\vect{y}) \Big) 
  \bigg],
  \label{eq:estimation_mid}
\end{align}}
\arxiv{
\begin{align}
  &\hspace{-2em}\EEy \exp\bigg\{\lambda \sn \scfunc\Big(\eta(\node{p},y_t) \node{v} \Big) \bigg\} \nonumber \\
  = \ & \EEyn{n-1} \mathop{\mathbb{E}}_{y_n \sim p_n(\vect{y})} \bigg[ \exp\bigg\{\lambda \sn \scfunc\Big(\eta(\node{p},y_t) \node{v} \Big) \bigg\} \bigg] \nonumber \\
  = \ & \EEyn{n-1} \bigg[ \exp\bigg\{\lambda \sum_{t=1}^{n-1} \scfunc\Big(\eta(\node{p},y_t) \node{v} \Big) \bigg\}
  \, \expfunc_{p_n(\vect{y}), \lambda}\Big(v_n(\vect{y}) \Big)
  \bigg],
  \label{eq:estimation_mid}
\end{align}}
where, for any $p \in [0,1]$ and $\lambda > 0$, we define $\expfunc_{p,\lambda} : [-1,1] \to \Reals$ by
\icml{
\begin{align*}
  \expfunc_{p,\lambda}(v)
  &= \EE_{y \sim p} \exp\Big\{\lambda \scfunc\Big(\eta(p,y) v \Big)\Big\} \\
  &= p \left(1 + \frac{\abs{v}}{p} \right)^\lambda \exp\left\{-\lambda\left(\frac{v+\abs{v}}{p} \right) \right\} \\ 
  &\hspace{1em} + (1-p) \left(1 + \frac{\abs{v}}{1-p} \right)^\lambda \exp\left\{\lambda\left(\frac{v-\abs{v}}{1-p} \right) \right\}.
\end{align*}}
\arxiv{
\begin{align*}
  \expfunc_{p,\lambda}(v)
  = \ & \EE_{y \sim p} \exp\Big\{\lambda \scfunc\Big(\eta(p,y) v \Big)\Big\} \\ 
  = \ & p \left(1 + \frac{\abs{v}}{p} \right)^\lambda \exp\left\{-\lambda\left(\frac{v+\abs{v}}{p} \right) \right\} 
  + (1-p) \left(1 + \frac{\abs{v}}{1-p} \right)^\lambda \exp\left\{\lambda\left(\frac{v-\abs{v}}{1-p} \right) \right\}.
\end{align*}}

Then, we observe the following.

\begin{lemma}\label{lem:psi-helper}
Whenever $\lambda \leq \frac{\log(3) - \log(2)}{2 - \log(2)}$,
\begin{align*}
  \sup_{p \in [0,1]} \sup_{v \in [p-1,p]} \expfunc_{p,\lambda}(v) \leq 1.
\end{align*}
\end{lemma}

The proof of \cref{lem:psi-helper} is a tedious calculation, and we leave it for Appendix~\ref{prf:main-extra}, but we provide a brief sketch of the argument here. First, $\expfunc_{p,\lambda}(v)$ can be simplified by fixing $v$ to be positive or negative. This allows us to show that if $\lambda$ is smaller than some function of $p$ and $v$, $\expfunc_{p,\lambda}(v)$ is increasing when $v<0$ and decreasing when $v>0$. Then, we observe that this function of $p$ and $v$ (which must upper bound $\lambda$) is lower bounded by $\frac{\log(3) - \log(2)}{2 - \log(2)}$. Finally, since $\expfunc_{p,\lambda}(0) = 1$ for any $p \in [0,1]$ and $\lambda > 0$, the result holds.

Thus, when $\lambda \leq \frac{\log(3) - \log(2)}{2 - \log(2)}$, (\ref{eq:estimation_mid}) is bounded above by 
\begin{align*}
  \EEyn{n-1} \exp\left\{\lambda \sum_{t=1}^{n-1} \scfunc\left(\eta(\node{p},y_t) \node{v} \right) \right\}.
\end{align*}
Iterating this argument through $t \in [n]$ and taking $\lambda$ as large as possible gives the result.
\end{proof}

We can apply \cref{lem:estimation} directly to (\ref{eq:estimation})
by observing that each tree $\vect{g} \in \Fpx$ can be written as
$\vect{p} - (f \circ \vect{x})$ for some $f \in \F$, and consequently
$\node{g} \in [\node{p} - 1, \node{p}]$ for all times $t \in [n]$ and
paths $\vect{y} \in \{0,1\}^n$. Thus, without loss of generality, any
cover $\Vpx$ of $\Fpx$ can also be assumed to satisfy $\node{v} \in
[\node{p} - 1, \node{p}]$, as clipping its value to this range will
only decrease the approximation error.

\cref{thm:main} now follows by applying
(\ref{eq:approximation_exp}) and (\ref{eq:eta_exp}) to
(\ref{eq:approximation}) and applying
\cref{lem:estimation} to (\ref{eq:estimation}).
\icml{\qed}
\arxiv{\qedarxiv}

\icml{\section{Proof of \cref{thm:lower}}}
\arxiv{\section{Proof of Theorem~\ref{thm:lower}}}
\label{sec:lower-proof}
We now prove \pref{thm:lower}. \cref{lem:lipschitz} in Appendix~\ref{prf:lower} shows that for $\F$ defined to be the $1$-Lipschitz experts on $[0,1]^p$,
$\seqentropy{\F}{\gamma}{n}{\infty} = \Theta(\gamma^{-p})$, so
(\ref{eqn:example}) applies for the upper bound. It remains to show
that the lower bound holds. To begin, we lower bound the minimax regret in
our adversarial setting by the minimax risk (the analogue of regret in batch learning) for the simpler
i.i.d.\ batch setting with a well-specified model, which admits a simple
expression in terms of KL divergence.

Let
$\hat f$ denote an arbitrary prediction strategy for the player that, for each $t$, outputs a predictor $\hat f_t:\context \to [0,1]$ using only the history $h_{1:t-1}$.
Then,
let
$\borel$ be the set of all distributions on $(\context, [0,1])$, and define the set
\icml{
\begin{align*}
  \borelF =
  \Big\{\dist \in \borel: \exists \optf \in \F \ \forall x \in \context \ \optf(x) = \! \! \mathop{\EE}_{(x, y) \sim \dist}\left[y \lvert x \right] \! \Big\}.
\end{align*}} 
\arxiv{
\begin{align*}
  \borelF =
  \Big\{\dist \in \borel: \exists \optf \in \F \ \forall x \in \context \ \optf(x) = \mathop{\EE}_{(x, y) \sim \dist}\left[y \lvert x \right] \Big\}.
\end{align*}} 
Using these new objects, and letting $\KL{p}{q}$ denote the KL divergence between $\berndist(p)$ and $\berndist(q)$, we obtain
the following result (proven in \cref{prf:regret_risk}).
\begin{lemma}\label{lem:regret_risk}
For any $\context$ and $\F \subseteq [0,1]^\context$,
\begin{align*}
  \frac{1}{n}\minimax{\F}
  \geq \inf_{\hat f} \sup_{\dist \in \borelF} \EE \Big[\KL{\optf(x)}{\hat f_n(x)}\Big],
\end{align*}
where $\EE$ denotes expectation over
$(x_{1:n-1}, y_{1:n-1}) \sim \dist^{\otimes n-1}$ and $(x,y) \sim \dist$. 
\end{lemma}
Thus, we have reduced the problem to lower-bounding the minimax risk for $\F$ under a well-specified model, which
is a more standard problem. To proceed, we use an argument along the lines of
Assouad's lemma \citep{assouad83estimation}, applied to our class $\F$ of $1$-Lipschitz functions on $[0,1]^p$.

First, fix $\eps \in (0,1/8)$, divide the space $[0,1]^p$ into $N =
(\textstyle{\frac{1}{4\eps}})^p$ bins of width $4\eps$, and without
loss of generality suppose that $N$ is an integer. Denote the centers
of each bin by $x\upp{1},\dots,x\upp{N}$. Define the set $\hcube =
\{\pm 1\}^N$ and the class $\lipF \subseteq \F$ as follows: for each
$v \in \hcube$, define the function $f_v$ such that $f_v(x\upp{i}) =
4\eps \1\{v_i = 1\} + \eps \1\{v_i = -1\}$ for $i \in [N]$. Define the
rest of $f_v$ by some linear interpolation between these points,
and observe that $f_v$ is $1$-Lipschitz. Finally, for any $v \in
\hcube$, define the distribution $\dist_v$ on $([0,1]^p, [0,1])$ by $x
\sim \unifdist(\{x\upp{1},\dots,x\upp{N}\})$ and $y \lvert x \sim
\berndist(f_v(x))$.

Pick $v \in \hcube$ and $f:[0,1]^p \to [0,1]$. By definition of $\dist_v$,
\begin{align*}
  \mathop{\EE}_{x \sim \dist_v} \KL{f_v(x)}{f(x)}
  = \frac{1}{N} \sum_{i=1}^N \KL{f_v(x\upp{i})}{f(x\upp{i})}.
\end{align*}
Next, we use \cref{lem:kl_eps} in Appendix~\ref{prf:lower} to lower
bound the KL divergence. Specifically, if $v_i = 1$ then $f_v(x\upp{i}) = 4\eps$, so
\begin{align*}
  \KL{f_v(x\upp{i})}{f(x\upp{i})}
  \geq \frac{2\eps}{3}\1\{f(x\upp{i}) \leq 2\eps\},
\end{align*}
and if $v_i = -1$ then $f_v(x\upp{i}) = \eps$, so
\begin{align*}
  \KL{f_v(x\upp{i})}{f(x\upp{i})}
  \geq \frac{\eps}{4}\1\{f(x\upp{i}) \geq 2\eps\}.
\end{align*}
That is, for all $i\in\brk*{N}$, 
\icml{
\begin{align*}
  &\hspace{-1em}\KL{f_v(x\upp{i})}{f(x\upp{i})} \\
  &\geq \frac{\eps}{4}\Big[\1\{v_i = 1 \, \land \, f(x\upp{i}) < 2\eps\} \\ 
  & \qquad \qquad + \1\{v_i = -1 \, \land \, f(x\upp{i}) \geq 2\eps\}\Big].
\end{align*}}
\arxiv{
\begin{align*}
  \KL{f_v(x\upp{i})}{f(x\upp{i})}
  &\geq \frac{\eps}{4}\Big[\1\{v_i = 1 \, \land \, f(x\upp{i}) < 2\eps\}
  + \1\{v_i = -1 \, \land \, f(x\upp{i}) \geq 2\eps\}\Big].
\end{align*}}
Now, since the expression in \pref{lem:regret_risk} involves the
supremum over all $\dist \in \borelF$, we can obtain a lower bound by taking an
expectation over $v$ uniformly chosen from $\hcube$ and setting
$\dist=\cD_{v}$. In particular, for each $i \in N$, we define the distributions $\dist^{\otimes n-1}_{+i} = 2^{-(N-1)} \sum_{v\in\hcube: v_i=1} \dist^{\otimes n-1}_v$ and $\dist^{\otimes n-1}_{-i} = 2^{-(N-1)} \sum_{v\in\hcube: v_i=-1} \dist^{\otimes n-1}_v$. 
Using the shorthand $\dist^{\otimes n-1}(\cdot)$ to denote $\PP_{(x_{1:n-1}, y_{1:n-1}) \sim \dist^{\otimes n-1}}(\cdot)$,  
we obtain the lower bound for any $\hat f$ of
\icml{
\begin{align*}
  &\sup_{\dist \in \borelF} \EE \Big[\KL{\optf(x)}{\hat f_{n}(x)}\Big] \\
  &\geq \frac{1}{2^N}\sum_{v \in \hcube} \EE \Big[\KL{f_v(x)}{\hat f_{n}(x)}\Big] \\
  &\geq \frac{1}{2^N}\sum_{v \in \hcube} \frac{\eps}{4N} \sum_{i=1}^N
  \Big[\1\{v_i = 1\}\dist^{\otimes n-1}_v(\hat f_{n}(x\upp{i}) < 2\eps) \\ 
  & \qquad \qquad \qquad \ \ \
  + \1\{v_i = -1\}\dist^{\otimes n-1}_v(\hat f_{n}(x\upp{i}) \geq 2\eps) \Big] \\
  &= \frac{\eps}{8N}\sum_{i=1}^N \Big[\dist^{\otimes n-1}_{+i}(\hat f_{n}(x\upp{i}) < 2\eps) \\
  & \qquad \qquad \qquad  + \dist^{\otimes n-1}_{-i}(\hat f_{n}(x\upp{i}) \geq 2\eps) \Big].
\end{align*}}
\arxiv{
\begin{align*}
  &\hspace{-2em}\sup_{\dist \in \borelF} \EE \Big[\KL{\optf(x)}{\hat f_{n}(x)}\Big] \\
  &\geq \frac{1}{2^N}\sum_{v \in \hcube} \EE \Big[\KL{f_v(x)}{\hat f_{n}(x)}\Big] \\
  &\geq \frac{1}{2^N}\sum_{v \in \hcube} \frac{\eps}{4N} \sum_{i=1}^N
  \Big[\1\{v_i = 1\}\dist^{\otimes n-1}_v(\hat f_{n}(x\upp{i}) < 2\eps)
  + \1\{v_i = -1\}\dist^{\otimes n-1}_v(\hat f_{n}(x\upp{i}) \geq 2\eps) \Big] \\
  &= \frac{\eps}{8N}\sum_{i=1}^N \Big[\dist^{\otimes n-1}_{+i}(\hat f_{n}(x\upp{i}) < 2\eps)
  + \dist^{\otimes n-1}_{-i}(\hat f_{n}(x\upp{i}) \geq 2\eps) \Big].
\end{align*}}
Then, we observe that for each $i \in [N]$,
\icml{
\begin{align*}
  &\dist^{\otimes n-1}_{+i}(\hat f_{n}(x\upp{i}) < 2\eps) + \dist^{\otimes n-1}_{-i}(\hat f_{n}(x\upp{i}) \geq 2\eps) \\
  &= 1 + \dist^{\otimes n-1}_{+i}(\hat f_{n}(x\upp{i}) < 2\eps) - \dist^{\otimes n-1}_{-i}(\hat f_{n}(x\upp{i}) < 2\eps) \\
  &\geq 1 - \abs{\dist^{\otimes n-1}_{+i}(\hat f_{n}(x\upp{i}) < 2\eps) - \dist^{\otimes n-1}_{-i}(\hat f_{n}(x\upp{i}) < 2\eps)} \\
  &\geq 1 - \Norm{\dist^{\otimes n-1}_{+i} - \dist^{\otimes n-1}_{-i}}_{\mathrm{TV}}.
\end{align*}}
\arxiv{
\begin{align*}
  \dist^{\otimes n-1}_{+i}(\hat f_{n}(x\upp{i}) < 2\eps) + \dist^{\otimes n-1}_{-i}(\hat f_{n}(x\upp{i}) \geq 2\eps) 
  = \ & 1 + \dist^{\otimes n-1}_{+i}(\hat f_{n}(x\upp{i}) < 2\eps) - \dist^{\otimes n-1}_{-i}(\hat f_{n}(x\upp{i}) < 2\eps) \\
  \geq \ & 1 - \abs{\dist^{\otimes n-1}_{+i}(\hat f_{n}(x\upp{i}) < 2\eps) - \dist^{\otimes n-1}_{-i}(\hat f_{n}(x\upp{i}) < 2\eps)} \\
  \geq \ & 1 - \Norm{\dist^{\otimes n-1}_{+i} - \dist^{\otimes n-1}_{-i}}_{\mathrm{TV}}.
\end{align*}}
Next, for each $v \in \hcube$, we define $\dist^{\otimes n-1}_{v, +i}$ to
be the distribution $\dist^{\otimes n-1}_v$ with $v_i$ forced to $1$,
and similarly define $\dist^{\otimes n-1}_{v, -i}$ to be the
distribution $\dist^{\otimes n-1}_v$ with $v_i$ forced to $-1$. Then,
following the standard argument, we observe that
\begin{align*}
  \Norm{\dist^{\otimes n-1}_{+i} - \dist^{\otimes n-1}_{-i}}_{\mathrm{TV}}
  &= \Norm{\frac{1}{2^N} \sum_{v \in \hcube} [\dist^{\otimes n-1}_{v, +i} - \dist^{\otimes n-1}_{v,-i}]}_{\mathrm{TV}} \\
  &\leq \frac{1}{2^N} \sum_{v \in \hcube} \Norm{\dist^{\otimes n-1}_{v, +i} - \dist^{\otimes n-1}_{v,-i}}_{\mathrm{TV}} \\
  &\leq \max_{v,i} \Norm{\dist^{\otimes n-1}_{v, +i} - \dist^{\otimes n-1}_{v,-i}}_{\mathrm{TV}}.
\end{align*}
Thus, we can apply this to \cref{lem:regret_risk} to get
\begin{align}
  \hspace{-0.5pt}\logminimax{\F} \geq n \frac{\eps}{8} \Big[1 - \max_{v,i} \Norm{\dist^{\otimes n-1}_{v, +i} - \dist^{\otimes n-1}_{v,-i}}_{\mathrm{TV}}\Big].
\label{eqn:lb-TV}
\end{align}
To further lower bound this, consider a fixed $v \in \hcube$ and $i
\in [N]$, and use $f_{v,+i}$ to denote $f_v$ with $v_i$ forced to $1$,
with the analogous definition for $f_{v, -i}$. By Pinsker's inequality and chain rule for KL,
\icml{
\begin{align*}
  &\hspace{-1em}\Norm{\dist^{\otimes n-1}_{v, +i} - \dist^{\otimes n-1}_{v,-i}}^2_{\mathrm{TV}} \\
  &\leq \frac{1}{2} \KL{\dist^{\otimes n-1}_{v, +i}}{\dist^{\otimes n-1}_{v,-i}} \\
  &= \frac{n-1}{2N} \sum_{j=1}^N \KL{f_{v,+i}(x\upp{j})}{f_{v,-i}(x\upp{j})} \\
  &= \frac{n-1}{2N}\cdot{}\KL{4\eps}{\eps},
\end{align*}}
\arxiv{
\begin{align*}
  \Norm{\dist^{\otimes n-1}_{v, +i} - \dist^{\otimes n-1}_{v,-i}}^2_{\mathrm{TV}}
  \leq \ & \frac{1}{2} \KL{\dist^{\otimes n-1}_{v, +i}}{\dist^{\otimes n-1}_{v,-i}} \\
  = \ & \frac{n-1}{2N} \sum_{j=1}^N \KL{f_{v,+i}(x\upp{j})}{f_{v,-i}(x\upp{j})} \\
  = \ & \frac{n-1}{2N}\cdot{}\KL{4\eps}{\eps},
\end{align*}}
where the last step uses that $f_{v,+i}$ and $f_{v,-i}$ agree everywhere except $x\upp{i}$. Finally, we observe that
\icml{
\begin{align*}
  \KL{4\eps}{\eps}
  &= 4\eps\log(4) + (1-4\eps)\log\left(\frac{1-4\eps}{1-\eps} \right) \\
  &\leq 4\eps\log(4) \\
  &\leq 8\eps.
\end{align*}}
\arxiv{
\begin{align*}
  \KL{4\eps}{\eps}
  &= 4\eps\log(4) + (1-4\eps)\log\left(\frac{1-4\eps}{1-\eps} \right) 
  \leq 4\eps\log(4) 
  \leq 8\eps.
\end{align*}}
We conclude from the definition of $N$ that
\icml{
\begin{align*}
  \Norm{\dist^{\otimes n-1}_{v, +i} - \dist^{\otimes n-1}_{v,-i}}^2_{\mathrm{TV}}
  &\leq 4 \frac{(n-1)\eps}{N} \\
  &= 4(n-1) \eps (4\eps)^p \\
  &\leq n (4\eps)^{1+p}. 
\end{align*}}
\arxiv{
\begin{align*}
  \Norm{\dist^{\otimes n-1}_{v, +i} - \dist^{\otimes n-1}_{v,-i}}^2_{\mathrm{TV}}
  \leq 4 \frac{(n-1)\eps}{N} 
  = 4(n-1) \eps (4\eps)^p 
  \leq n (4\eps)^{1+p}. 
\end{align*}}
Setting $\eps = \textstyle{\frac{1}{8}}n^{-\frac{1}{p+1}}$ gives $n (4\eps)^{1+p} = 2^{-(1+p)} \leq 1/4$, and plugging this into (\ref{eqn:lb-TV}) gives the lower bound
\icml{
\begin{align*}
  \qquad \ \ \ \
  \logminimax{\F}
  \geq n \frac{n^{-\frac{1}{p+1}}}{(8)(8)}[1-1/2]
  = \frac{n^{\frac{p}{p+1}}}{128}.
  \qquad \ \ \  \qed
\end{align*}}
\arxiv{
\begin{align*}
  \qquad \qquad \qquad \qquad \qquad \qquad \qquad \qquad \
  \logminimax{\F}
  \geq n \frac{n^{-\frac{1}{p+1}}}{(8)(8)}[1-1/2]
  = \frac{n^{\frac{p}{p+1}}}{128}.
  \qquad \qquad \qquad \qquad \qquad \qquad \ \ \ \qed
\end{align*}}

\section{Discussion}
\label{sec:discussion}

We have shown that the self-concordance property of log
loss leads to improved bounds on the minimax regret for sequential
probability assignment with rich classes of experts, and that the rates we provide cannot be
improved further without stronger structural assumptions on the expert
class. An important open problem is to develop more refined complexity
measures (e.g., variants of sequential entropy tailored directly to
the log loss rather than the $\Linf$ norm) that lead to matching
upper and lower bounds for \emph{all} classes of experts; we intend to
pursue this in future work.

On the technical side, it would be interesting to extend our
guarantees to infinite outcome spaces; that is, adversarial
online density estimation. To the best of our knowledge, very little progress has been made on this problem without stochastic assumptions.

\icml{\newpage}
\subsection*{Acknowledgements}

BB is supported by an NSERC Canada Graduate Scholarship. The authors thank Jeffrey Negrea for many helpful discussions. DF thanks Sasha Rakhlin and Karthik Sridharan for many useful discussions over the years about this problem. This work was begun while all three authors were visiting the Simons Institute for Theoretical Computing in Berkeley, California for the Foundations of Deep Learning program, and completed while BB and DMR were visiting the Institute for Advanced Study in Princeton, New Jersey for the Special Year on Statistics, Optimization, and Theoretical Machine Learning. DMR acknowledges the support of an NSERC Discovery Grant and a stipend provided by the Charles Simonyi Endowment. BB's travel to the Institute for Advanced Study was funded by an NSERC Michael Smith Foreign Study Supplement. DF acknowledges the support of TRIPODS award \#1740751.

\icml{
  \bibliographystyle{icml2020}
}
\bibliography{logloss,refs}
\newpage
\onecolumn

\icml{
\begin{center}
\Large \textbf{Supplementary Material}
\end{center}
}

\appendix

\icml{\section{Additional Details for Proof of \cref{thm:main}}}
\arxiv{\section{Additional Details for Proof of Theorem~\ref{thm:main}}}

\icml{\subsection{Proof of \cref{lem:dual}}}
\arxiv{\subsection{Proof of Lemma~\ref{lem:dual}}}
\label{prf:dual}
\newcommand{\predint}{\cI_{\delta}}

The proof follows similarly to that of \citet{abernethy09duality} and
\citet{rakhlin15learning}, but since we require a variant for unbounded losses we work out the details here for completeness. To keep notation compact, we adopt the
repeated operator notation from \citet{rakhlin12notes}, using
$\game{\textsf{Op}_t}{n}[\cdots]$ to denote $\textsf{Op}_1
\textsf{Op}_2 \cdots \textsf{Op}_n [\cdots]$.

To begin, let us assume for simplicity that for every sequence
$x_{1:n}, y_{1:n}$,
$\inf_{f\in\cF}\sum_{t=1}^{n}\ls(f(x_t),y_t)<\infty$; note that this
can always be made to hold by adding the constant $1/2$ function to
$\cF$, and this only increases the sequential entropy by an additive constant.

Next, let us move to an upper bound by restricting the player's predictions to the interval $\predint\ldef{}\brk*{\delta,1-\delta}$, where
$0<\delta\leq{}1/2$. Then, we may write
\begin{align*}
  \minimax{\F}
  &\leq{}
  \game{\sup_{x_t} \inf_{\smash{\pred_t \in \predint}} \sup_{\smash{y_t \in \{0,1\}}}}{n}
  \bigg[
    \sum_{t=1}^n \logloss(\pred_t, y_t) 
    - \inf_{f \in \F} \sum_{t=1}^n \logloss(f(x_t), y_t) 
  \bigg] \\
    &= 
  \game{\sup_{x_t} \inf_{\smash{\pred_t \in \predint}} \sup_{\smash{y_t \in \{0,1\}}}}{n-1}
  \sup_{x_n} \inf_{\smash{\pred_n \in \predint}} \sup_{\smash{p_n \in [0,1]}} \mathop{\EE}_{y_n \sim p_n}
  \bigg[
    \sum_{t=1}^n \logloss(\pred_t, y_t) 
    - \inf_{f \in \F} \sum_{t=1}^n \logloss(f(x_t), y_t) 
  \bigg].
\end{align*}
We now wish to apply a minimax theorem to the function 
\begin{align*}
  A(\pred_n, p_n) 
  = \mathop{\EE}_{y_n \sim p_n}
    \bigg[
      \sum_{t=1}^n \logloss(\pred_t, y_t) 
      - \inf_{f \in \F} \sum_{t=1}^n \logloss(f(x_t), y_t) 
    \bigg].
\end{align*}
We appeal to a basic variant of von Neumann's minimax theorem.
\begin{theorem}[\citealt{sion1958minimax}]
  \label{thm:sion}
  Let $X$ be a convex, compact subset of a linear topological space
  and $Y$ be a compact subset of a linear topological space. Let
  $f:X\times{}Y\to\bbR$. Suppose that $f(x,\cdot)$ is
  upper-semicontinuous and quasiconcave for all $x\in{}X$ and
  $f(\cdot,y)$ is lower-semicontinuous and quasiconvex for all
  $y\in{}Y$. Then
  \[
    \inf_{x\in{}X}\sup_{y\in{}Y}f(x,y) = \sup_{y\in{}Y}\inf_{x\in{}X} f(x,y).
  \]
\end{theorem}
To apply this result, we take $X=\predint$ and $Y=\brk*{0,1}$, both of
which are convex and compact. We observe that $A(\phat_n,p_n)$ depends
on $\phat_n$ only through $\En_{y\sim{}p_n}\ls(\phat_n,p_n)$, which is
convex and continuous over $\predint$. Moreover $A(\phat_n,p_n)$ is a
bounded linear function of $p_n$ over $\brk*{0,1}$, and hence is
concave and continuous. Thus, the theorem applies, and we
have
\icml{
\begin{align*}
  \minimax{\F}
  &\leq 
  \game{\sup_{x_t} \inf_{\smash{\pred_t \in \predint}} \sup_{\smash{y_t \in \{0,1\}}}}{n-1}
  \sup_{x_n} \sup_{\smash{p_n \in [0,1]}} \inf_{\smash{\pred_n \in \predint}} \mathop{\EE}_{y_n \sim p_n}
  \bigg[
    \sum_{t=1}^n \logloss(\pred_t, y_t) 
    - \inf_{f \in \F} \sum_{t=1}^n \logloss(f(x_t), y_t) 
  \bigg] \\
  &=
  \game{\sup_{x_t} \inf_{\smash{\pred_t \in \predint}} \sup_{\smash{y_t \in \{0,1\}}}}{n-1} 
  \bigg[
    \sum_{t=1}^{n-1} \logloss(\pred_t,y_t)
    + 
    \sup_{x_n} \sup_{\smash{p_n \in [0,1]}}
    \Big[
    \inf_{\smash{\pred_n \in \predint}} \mathop{\EE}_{y_n \sim p_n} \logloss(\pred_n, y_n)
    - \mathop{\EE}_{y_n \sim p_n} \inf_{f \in \F} \sum_{t=1}^n \logloss(f(x_t), y_t) 
    \Big]
  \bigg] \\
  &= \game{\sup_{x_t} \inf_{\smash{\pred_t \in \predint}} \sup_{\smash{y_t \in \{0,1\}}}}{n-2}
  \sup_{x_{n-1}} \inf_{\smash{\pred_{n-1} \in \predint}} \sup_{\smash{p_{n-1} \in [0,1]}} \mathop{\EE}_{y_{n-1} \sim p_{n-1}} \\[-2pt]
  & \qquad \qquad \bigg[
    \sum_{t=1}^{n-1} \logloss(\pred_t,y_t)
    + 
    \sup_{x_n} \sup_{\smash{p_n \in [0,1]}}
    \Big[
    \inf_{\smash{\pred_n \in \predint}} \mathop{\EE}_{y_n \sim p_n} \logloss(\pred_n, y_n)
    - \mathop{\EE}_{y_n \sim p_n} \inf_{f \in \F} \sum_{t=1}^n \logloss(f(x_t), y_t) 
    \Big]
  \bigg].
\end{align*}}
\arxiv{
\begin{align*}
  \hspace{-2pt}\minimax{\F}
  &\leq
  \game{\sup_{x_t} \inf_{\smash{\pred_t \in \predint}} \sup_{\smash{y_t \in \{0,1\}}}}{n-1}
  \sup_{x_n} \sup_{\smash{p_n \in [0,1]}} \inf_{\smash{\pred_n \in \predint}} \mathop{\EE}_{y_n \sim p_n}
  \bigg[
    \sum_{t=1}^n \logloss(\pred_t, y_t) 
    - \inf_{f \in \F} \sum_{t=1}^n \logloss(f(x_t), y_t) 
  \bigg] \\
  &=
  \game{\sup_{x_t} \inf_{\smash{\pred_t \in \predint}} \sup_{\smash{y_t \in \{0,1\}}}}{n-1} \hspace{-2pt}
  \bigg[
    \sum_{t=1}^{n-1} \logloss(\pred_t,y_t)
    + 
    \sup_{x_n} \sup_{\smash{p_n \in [0,1]}}\hspace{-2pt}
    \Big[
    \inf_{\smash{\pred_n \in \predint}} \mathop{\EE}_{y_n \sim p_n} \hspace{-2pt} \logloss(\pred_n, y_n)
    - \hspace{-2pt} \mathop{\EE}_{y_n \sim p_n} \inf_{f \in \F} \sum_{t=1}^n \logloss(f(x_t), y_t) 
    \Big]\hspace{-2pt}
  \bigg] \\
  &= \game{\sup_{x_t} \inf_{\smash{\pred_t \in \predint}} \sup_{\smash{y_t \in \{0,1\}}}}{n-2}
  \sup_{x_{n-1}} \inf_{\smash{\pred_{n-1} \in \predint}} \sup_{\smash{p_{n-1} \in [0,1]}} \mathop{\EE}_{y_{n-1} \sim p_{n-1}} \\[-2pt]
  & \qquad \qquad \bigg[
    \sum_{t=1}^{n-1} \logloss(\pred_t,y_t)
    + 
    \sup_{x_n} \sup_{\smash{p_n \in [0,1]}}
    \Big[
    \inf_{\smash{\pred_n \in \predint}} \mathop{\EE}_{y_n \sim p_n} \logloss(\pred_n, y_n)
    - \mathop{\EE}_{y_n \sim p_n} \inf_{f \in \F} \sum_{t=1}^n \logloss(f(x_t), y_t) 
    \Big]
  \bigg].
\end{align*}}

Once again, we wish to apply the minimax theorem, but this time to the function
\begin{align*}
  B(\pred_{n-1}, p_{n-1}) 
  = \mathop{\EE}_{y_{n-1} \sim p_{n-1}}
  \bigg[
    \sum_{t=1}^{n-1} \logloss(\pred_t,y_t)
    + 
    \sup_{x_n} \sup_{\smash{p_n \in [0,1]}}
    \Big[
    \inf_{\smash{\pred_n \in \predint}} \mathop{\EE}_{y_n \sim p_n} \logloss(\pred_n, y_n)
    - \mathop{\EE}_{y_n \sim p_n} \inf_{f \in \F} \sum_{t=1}^n \logloss(f(x_t), y_t) 
    \Big]
  \bigg].
\end{align*}
The same logic applies, where we observe that $B$ is a bounded linear
function in
$p_{n-1}$ and only depends on $\pred_{n-1}$ through
$\logloss(\pred_{n-1},y_{n-1})$, so the convexity and continuity of
log loss over $\cI_{\delta}$ suffices. That is,
\begin{align*}
  \minimax{\F}
  &\leq \game{\sup_{x_t} \inf_{\smash{\pred_t \in \predint}} \sup_{\smash{y_t \in \{0,1\}}}}{n-2}
  \sup_{x_{n-1}} \sup_{\smash{p_{n-1} \in [0,1]}} \inf_{\smash{\pred_{n-1} \in \predint}} \mathop{\EE}_{y_{n-1} \sim p_{n-1}} \\[-2pt]
  & \qquad \qquad \bigg[
    \sum_{t=1}^{n-1} \logloss(\pred_t,y_t)
    + 
    \sup_{x_n} \sup_{\smash{p_n \in [0,1]}}
    \bigg[
    \inf_{\smash{\pred_n \in \predint}} \mathop{\EE}_{y_n \sim p_n} \logloss(\pred_n, y_n)
    - \mathop{\EE}_{y_n \sim p_n} \inf_{f \in \F} \sum_{t=1}^n \logloss(f(x_t), y_t) 
    \bigg]
  \bigg] \\
  &= \game{\sup_{x_t} \inf_{\smash{\pred_t \in \predint}} \sup_{\smash{y_t \in \{0,1\}}}}{n-3}
  \sup_{x_{n-2}} \inf_{\smash{\pred_{n-2} \in \predint}} \sup_{\smash{p_{n-2} \in [0,1]}} \mathop{\EE}_{y_{n-2} \sim p_{n-2}} \\[-2pt]
  & \qquad \qquad \Bigg[
    \sum_{t=1}^{n-2} \logloss(\pred_t,y_t)
    +
    \sup_{x_{n-1}} \sup_{\smash{p_{n-1} \in [0,1]}}
    \Bigg[
    \inf_{\smash{\pred_{n-1} \in \predint}} \mathop{\EE}_{y_{n-1} \sim p_{n-1}} \logloss(\pred_{n-1}, y_{n-1}) \\[-4pt]
    &\qquad \qquad \qquad  
    + \mathop{\EE}_{y_{n-1} \sim p_{n-1}}
    \sup_{x_n} \sup_{\smash{p_n \in [0,1]}}
    \bigg[
    \inf_{\smash{\pred_n \in \predint}} \mathop{\EE}_{y_n \sim p_n} \logloss(\pred_n, y_n)
    - \mathop{\EE}_{y_n \sim p_n} \inf_{f \in \F} \sum_{t=1}^n \logloss(f(x_t), y_t) 
    \bigg]
    \Bigg]
  \Bigg].
\end{align*}
Collecting terms and iterating the argument down through all $n$
rounds gives
\[
  \minimax{\cF} \leq{}
  \dtri*{\sup_{x_t}\sup_{p_t\in\brk*{0,1}}\mathop{\EE}_{y_t\sim{}p_t}}_{t=1}^{n}
  \sup_{f\in\cF}\brk*{\sum_{t=1}^{n}\inf_{\pred_t \in \predint} \mathop{\EE}_{y_t \sim p_t} [\logloss(\pred_t, y_t)] - \logloss(f(x_t), y_t)}.
\]
Applying \pref{lem:clipping}, this is bounded by
\[
  \minimax{\cF} \leq{}
  \dtri*{\sup_{x_t}\sup_{p_t\in\brk*{0,1}}\mathop{\EE}_{y_t\sim{}p_t}}_{t=1}^{n}
  \sup_{f\in\cF}\brk*{\sum_{t=1}^{n}\inf_{\pred_t \in \brk*{0,1}}
    \mathop{\EE}_{y_t \sim p_t} [\logloss(\pred_t, y_t)] -
    \logloss(f(x_t), y_t)} + 2\delta{}n.
\]
Since the right-hand side only depends on $\delta$ through the error
term $2\delta{}n$, we can take the limit as $\delta\to{}0$ to get
\[
  \minimax{\cF} \leq{}
  \dtri*{\sup_{x_t}\sup_{p_t\in\brk*{0,1}}\mathop{\EE}_{y_t\sim{}p_t}}_{t=1}^{n}
  \sup_{f\in\cF}\brk*{\sum_{t=1}^{n}\inf_{\pred_t \in \brk*{0,1}}
    \mathop{\EE}_{y_t \sim p_t} [\logloss(\pred_t, y_t)] -
    \logloss(f(x_t), y_t)}.
\]
The inequality in the other direction holds trivially by the max-min inequality, so we conclude equality.
\icml{\qed}
\arxiv{\qedarxiv}

\begin{lemma}\label{lem:clipping}
  For any $p\in\brk*{0,1}$ and $\delta\in\brk*{0,1/2}$, define
  \[
    p_{\delta} = \left\{
      \begin{array}{ll}
        \delta,&\quad{}p<\delta,\\
        p,&\quad{}p\in\brk*{\delta,1-\delta},\\
        1-\delta,&\quad{}p>1-\delta.
      \end{array}
      \right.
    \]
    Then for all $y\in\crl{0,1}$,
    $\ls(p_{\delta},y)\leq{}\ls(p,y)+2\delta$.
  \end{lemma}
  \begin{proof}
    If $y=1$, we have
    $\ls(p_{\delta},y)-\ls(p,y)=\log(p/p_{\delta})$. If
    $p\leq{}1-\delta$, this is at most zero. Otherwise, we have
    \[
      \log(p/p_{\delta}) = \log\prn*{\frac{p}{1-\delta}} =
      \log\prn*{1+\frac{p-(1-\delta)}{1-\delta}}
      \leq{} \frac{p-(1-\delta)}{1-\delta}\leq{}2\delta,
    \]
    where the last inequality uses that $1-\delta \geq 1/2$ and
    $p-(1-\delta)\leq{}\delta$.

    If $y=0$, we have
    $\ls(p_{\delta},y)-\ls(p,y)=\log\prn*{\frac{1-p}{1-p_{\delta}}}$,
    and the only non-trivial case is where $p<\delta$, where
    \[
      \log\prn*{\frac{1-p}{1-p_{\delta}}} =
      \log\prn*{\frac{1-p}{1-\delta}}
      =       \log\prn*{1 + \frac{\delta-p}{1-\delta}}
      \leq{} \frac{\delta-p}{1-\delta}\leq{}2\delta.
    \]
  \end{proof}

\icml{\subsection{Additional Details for Proof of \pref{lem:sc}}}
\arxiv{\subsection{Additional Details for Proof of Lemma~\ref{lem:sc}}}
  \begin{lemma}
    \label{lem:sc_edge_case}
     For the same setting as \pref{lem:sc}, the inequality
     (\ref{eq:sc_inequality}) holds almost surely when either $f\in\crl*{0,1}$ or $p\in\crl*{0,1}$.
\end{lemma}
\begin{proof}
  We first observe that the desired inequality may be written as
\icml{
  \begin{align}
    \label{eq:sc_eq}
     &\hspace{-2em}y\log\prn*{\frac{f}{p}} + (1-y)\log\prn*{\frac{1-f}{1-p}} \nonumber \\
 &\leq{} \prn*{\frac{-y}{p} + \frac{1-y}{1-p}}(p-f) -
\prn*{\frac{y}{p} + \frac{1-y}{1-p}}\abs{p-f}
+ \log\prn*{1
+ \prn*{\frac{y}{p} + \frac{1-y}{1-p}}\abs{p-f}}.
\end{align}}
\arxiv{
  \begin{align}
    \label{eq:sc_eq}
     &\hspace{-2em}y\log\prn*{\frac{f}{p}} + (1-y)\log\prn*{\frac{1-f}{1-p}} \nonumber \\
 &\leq{} \prn*{\frac{-y}{p} + \frac{1-y}{1-p}}(p-f) -
\prn*{\frac{y}{p} + \frac{1-y}{1-p}}\abs{p-f}
+ \log\prn*{1
+ \prn*{\frac{y}{p} + \frac{1-y}{1-p}}\abs{p-f}}.
\end{align}}
First, observe that if $p=1$, we must have $y=1$ almost
surely. Furthermore, since $y=1$, we may restrict to $f\in(0,1]$, as
the left-hand side above approaches $-\infty$
for $f\to{}0$. For all $f\in(0,1]$, the inequality simplifies to
\[
 \log(f) \leq{} \log(2-f) - 2(1-f).
 \]
 This can be seen to hold by observing that $\log(f) - \log(2-f) =
 \log\prn*{1- 2\frac{1-f}{2-f}} \leq{}-2(1-f)$, where we have used that
 $\log(1+x)\leq{}\frac{2x}{2+x}$ for $x\in(-1,0]$.

 Next, we similarly observe that for $p=0$, we may take $y=0$ and
 $f\in[0,1)$, and (\ref{eq:sc_eq}) simplifies to
 \[
 \log(1-f) \leq{} \log(1+f)-2f,
\]
 which follows from the same elementary inequality.

 Next, suppose that $p\in(0,1)$ and either $f=0$ or $f=1$. In this
 case, we may take $y=0$ or $y=1$ respectively, or else
 (\ref{eq:sc_eq}) is trivial. By direct calculation, we can verify
 that (\ref{eq:sc_eq}) holds with equality in both cases.
\end{proof}

\icml{\subsection{Proof of \cref{lem:approximation}}}
\arxiv{\subsection{Proof of Lemma~\ref{lem:approximation}}}
\label{prf:approximation}

First, observe that
\begin{align*}
  \scfunc(s) - \scfunc(t)
  = (s-t) - (\abs{s} - \abs{t}) + \log\left(1+\frac{\abs{s}-\abs{t}}{1+\abs{t}}\right).
\end{align*}
There are two cases to consider. If $\abs{t} < \abs{s}$, then since $\log(1+z) \leq z$ for all $z>-1$,
\begin{align*}
  \abs{t} - \abs{s} + \log\left(1+\frac{\abs{s}-\abs{t}}{1+\abs{t}}\right) 
  \leq \abs{t} - \abs{s} + \frac{\abs{s}-\abs{t}}{1+\abs{t}} 
  \leq \abs{t} - \abs{s} + \abs{s}-\abs{t} 
  = 0 
  \leq \abs{s-t}.
\end{align*}
Otherwise, if $\abs{s} \leq \abs{t}$, since $\abs{\abs{a}-\abs{b}} \leq \abs{a-b}$ for all $a,b \in \Reals$,
\begin{align*}
  \abs{t} - \abs{s} + \log\left(1+\frac{\abs{s}-\abs{t}}{1+\abs{t}}\right) 
  \leq \abs{t} - \abs{s} + \log(1) 
  \leq \abs{s-t}.
\end{align*}
Trivially, $(s-t) \leq \abs{s-t}$, which completes the proof.
\icml{\qed}
\arxiv{\qedarxiv}

\icml{\subsection{Proof of \cref{lem:psi-helper}}}
\arxiv{\subsection{Proof of Lemma~\ref{lem:psi-helper}}}
\label{prf:main-extra}

Recall that
\begin{align*}
  \expfunc_{p,\lambda}(v)
  &= \EE_{y \sim p} \exp\Big\{\lambda \scfunc\Big(\eta(p,y) v \Big)\Big\} \\
  &= p \left(1 + \frac{\abs{v}}{p} \right)^\lambda \exp\left\{-\lambda\left(\frac{v+\abs{v}}{p} \right) \right\}
  + (1-p) \left(1 + \frac{\abs{v}}{1-p} \right)^\lambda \exp\left\{\lambda\left(\frac{v-\abs{v}}{1-p} \right) \right\}.
\end{align*}
We now prove that for $\lambda \leq \frac{\log(3) - \log(2)}{2 - \log(2)}$,
\begin{align*}
  \sup_{p \in [0,1]} \sup_{v \in [p-1,p]} \expfunc_{p,\lambda}(v) \leq 1.
\end{align*}
Clearly, for any $p \in [0,1]$ and $\lambda >0$,
$\expfunc_{p,\lambda}(0) = 1$. We claim that there is a choice for
$\lambda$ which does not depend on $p$ for which the point $0$ in fact attains
the maximum over all $v\in\brk*{p-1,p}$. To see this, consider $\expfunc'_{p,\lambda}(v) = \frac{d}{dv} \expfunc_{p,\lambda}(v)$. We will show that there is some $\lambda >0$ such that for all $p \in [0,1]$, $\expfunc'_{p,\lambda}(v) \geq 0$ for $v \in [p-1, 0)$ and $\expfunc'_{p,\lambda}(v) \leq 0$ for $v \in (0, p]$, which suffices since $\expfunc_{p,\lambda}(v)$ is continuous in $v$.

First, we handle the edge cases.

Suppose $p=0$ and $\lambda > 0$. Then, for $v \in [-1,0]$,
\begin{align*}
  \expfunc_{0,\lambda}(v) 
  = (1+\abs{v})^\lambda e^{\lambda(v-\abs{v})}
  = (1-v)^\lambda e^{2\lambda v}.
\end{align*}
It remains to show $\expfunc_{0,\lambda}(v)$ is increasing on $v \in [-1,0]$, which follows since
\begin{align*}
  \expfunc'_{0,\lambda}(v) 
  = -\lambda (1-v)^{\lambda-1}e^{2\lambda v} + 2\lambda(1-v)^\lambda e^{2\lambda v}
  \geq -\lambda (1-v)^{\lambda-1}e^{2\lambda v} + 2\lambda(1-v)^{\lambda-1} e^{2\lambda v}
  \geq 0,
\end{align*}
where we have used that $(1-v)^{\lambda}\geq{}(1-v)^{\lambda{}-1}$,
which holds for all $\lambda\geq{}0$ since $1-v\geq{}1$. 

Next, suppose $p=1$ and $\lambda>0$. Then, for $v \in [0,1]$,
\begin{align*}
  \expfunc_{1,\lambda}(v) 
  = (1+\abs{v})^\lambda e^{\lambda(v+\abs{v})}
  = (1+v)^\lambda e^{2\lambda v}.
\end{align*}
We now wish to show $\expfunc_{1,\lambda}(v)$ is decreasing on $v \in [0,1]$, which follows since
\begin{align*}
  \expfunc'_{1,\lambda}(v)
  = \lambda(1+v)^{\lambda-1}e^{-2\lambda v} - 2\lambda(1+v)^\lambda e^{-2\lambda v}
  \leq \lambda(1+v)^{\lambda}e^{-2\lambda v} - 2\lambda(1+v)^\lambda e^{-2\lambda v}
  \leq 0,
\end{align*}
where we have similarly used that
$(1+v)^{\lambda-1}\leq{}(1+v)^{\lambda}$ whenever $\lambda,v\geq{}0$.

Thus, we can now fix $p \in (0,1)$.
First, consider $v \in [p-1,0)$. Then,
\begin{align*}
  \expfunc_{p,\lambda}(v) = p \left(1 - \frac{v}{p} \right)^\lambda + (1-p) \left(1 - \frac{v}{1-p} \right)^\lambda \exp\left\{\left(\frac{2\lambda v}{1-p} \right) \right\},
\end{align*}
where we have used that $\abs*{v}=-v$ to simplify. Thus,
\begin{align*}
  \expfunc'_{p,\lambda}(v)
  &= p\lambda \left(1-\frac{v}{p} \right)^{\lambda-1} \left(-\frac{1}{p}\right) 
    + (1-p)\lambda \left(1-\frac{v}{1-p} \right)^{\lambda-1} \left(-\frac{1}{1-p}\right) \exp\left\{\left(\frac{2\lambda v}{1-p} \right) \right\} \\
  & \qquad + (1-p) \left(1 - \frac{v}{1-p} \right)^\lambda \exp\left\{\left(\frac{2\lambda v}{1-p} \right) \right\} \frac{2\lambda}{1-p} \\
  &= \lambda \left[\left(\frac{1-p-v}{1-p} e^{\frac{2v}{1-p}}\right)^\lambda \left(2 - \frac{1-p}{1-p-v} \right) - \left(\frac{p-v}{p} \right)^{\lambda-1} \right].
\end{align*}
That is,
\begin{align}
  &  \expfunc'_{p,\lambda}(v) \geq 0 \nonumber \\
  &\iff\lambda \left[\left(\frac{1-p-v}{1-p} e^{\frac{2v}{1-p}}\right)^\lambda \left(2 - \frac{1-p}{1-p-v} \right) - \left(\frac{p-v}{p} \right)^{\lambda-1} \right] \geq 0 \nonumber \\
  &\iff\left(\frac{1-p-v}{1-p} e^{\frac{2v}{1-p}}\right)^\lambda \left(\frac{1-p-2v}{1-p-v} \right) \geq \left(\frac{p-v}{p} \right)^{\lambda-1} \nonumber \\
  &\iff\left(\frac{p(1-p-v)}{(1-p)(p-v)} e^{\frac{2v}{1-p}}\right)^\lambda \geq \frac{p(1-p-v)}{(1-p-2v)(p-v)},
  \label{eq:case1_iff}
\end{align}
where we have used that $p-v$, $1-p-v$, and $1-p-2v$ are all
positive. 
Now, we wish to be able to rearrange
this expression to extract a sufficient condition for $\lambda$. To do so, we need to check the sign of the terms to determine which way the inequality changes.
\begin{align*}
  \frac{d}{dv} \frac{p(1-p-v)}{(1-p)(p-v)} e^{\frac{2v}{1-p}}
  &= \frac{p}{1-p} \left[\frac{2}{1-p}\frac{1-p-v}{p-v} e^{\frac{2v}{1-p}} + \frac{1-p-v}{(p-v)^2} e^{\frac{2v}{1-p}} - \frac{1}{p-v} e^{\frac{2v}{1-p}} \right] \\
  &= \frac{p}{(1-p)(p-v)} e^{\frac{2v}{1-p}} \left[\frac{2(1-p-v)}{1-p} + \frac{1-p-v}{p-v} - 1 \right] \\
  &= \frac{p}{(1-p)(p-v)} e^{\frac{2v}{1-p}} \left[1 -\frac{2v}{1-p} + \frac{1-p-v}{p-v} \right] \\
  &\geq 0.
\end{align*}
The final inequality follows since $v < 0$. So, since this term is
increasing as $v$ increases, and at $v=0$ it takes the value 1, we
have $\frac{p(1-p-v)}{(1-p)(p-v)} e^{\frac{2v}{1-p}} \leq 1$. Thus, taking logarithms on both sides of
(\ref{eq:case1_iff}), we conclude that for all $p \in (0,1)$ and $v\in[p-1,0)$,
\begin{align}
  \expfunc'_{p,\lambda}(v) \geq 0 
   \iff
  \lambda \leq \frac{\log(p) + \log(1-p-v) - \log(p-v) - \log(1-p-2v)}{\log(p) + \log(1-p-v) - \log(p-v) - \log(1-p) + \frac{2v}{1-p}}.
  \label{eq:case1_lambda_condition}
\end{align}
Next, we show that the RHS of (\ref{eq:case1_lambda_condition})
admits a lower bound independent of $p$ and $v$, so that making
$\lambda$ small enough will always give us $\expfunc'_{p,\lambda}(v)
\geq 0$. To do so, we write the RHS above as a ratio of functions
$N_p(v)/D_p(v)$, where $N_p(v)$ denotes the numerator in
\eqref{eq:case1_lambda_condition} and $D_p(v)$ denotes the denominator. Observe that $\log(p) < \log(p-v)$ and
$\log(1-p-v) < \log(1-p-2v)$, so $N_p(v)<0$. Similarly, since $\log(1+x) \leq x$ for $x > 0$ and $-\frac{v}{1-p} > 0$,  
\begin{align*}
\log(1-p-v)-\log(1-p)+\frac{2v}{1-p} = \log\Big(1 +
\frac{-v}{1-p}\Big) + \frac{2v}{1-p} \leq -\frac{v}{1-p} + \frac{2v}{1-p}
= \frac{v}{1-p} < 0,
\end{align*}
which implies $D_p(v)<0$. 
Now, differentiating the numerator,
\begin{align*}
  N'_p(v) 
  = \frac{-1}{1-p-v} + \frac{1}{p-v} + \frac{2}{1-p-2v}
  = \frac{1-p}{(1-p-v)(1-p-2v)} + \frac{1}{p-v}
  > 0.
\end{align*}
Similarly, differentiating the denominator gives
\begin{align*}
  D'_p(v) 
  = \frac{-1}{1-p-v} + \frac{1}{p-v} + \frac{2}{1-p}
  = \frac{1-p-2v}{(1-p-v)(1-p)} + \frac{1}{p-v}
  > 0.
\end{align*}
In particular, we see that $N_p'(v)\leq{}D_p'(v)$, since
$\frac{2}{1-p-2v}\leq{}\frac{2}{1-p}$ when $v\leq{}0$. Further,
\begin{align*}
  N''_p(v) 
  = \frac{-1}{(1-p-v)^2} + \frac{1}{(p-v)^2} + \frac{4}{(1-p-2v)^2}
  = \frac{3(1-p)^2 - 4v(1-p)}{(1-p-v)^2(1-p-2v)^2} + \frac{1}{(p-v)^2}
  > 0,
\end{align*}
and
\begin{align*}
  D''_p(v) 
  = \frac{-1}{(1-p-v)^2} + \frac{1}{(p-v)^2}
  < N''_p(v).
\end{align*}
We will now apply the following elementary fact.
\begin{lemma}\label{lem:func_compare}
Let $f_1$ and $f_2$ be two nonnegative, twice differentiable functions defined on $(-\infty,0]$ with $f_1(0)=f_2(0)=0$. If $f_2'(x) \leq f_1'(x)$, $f_1''(x) \leq f_2''(x)$, and $f_1''(x) \leq 0$ for all $x \leq 0$, 
then $f_1(x)/f_2(x)$ is increasing on $(-\infty,0]$.
\end{lemma}
\begin{proof}
First, observe that for any $x \leq 0$,
\begin{align*}
  f_1(x) 
  = - \int_x^0 f'_1(x) dx
  \leq - \int_x^0 f'_2(x) dx
  = f_2(x).
\end{align*}
Then, we wish to show that $\frac{d}{dx} \frac{f_1(x)}{f_2(x)}\geq 0$ for all $x \leq 0$. By quotient rule, it suffices to show that $f_1'(x)f_2(x) \geq f_1(x) f_2'(x)$. Since $f_1'(0)f_2(0) = 0 = f_1(0)f_2'(0)$, we only need to show that $f_1'(x)f_2(x) - f_1(x) f_2'(x)$ is decreasing on $(-\infty,0]$, which we do by showing it has negative derivative. In particular,
\begin{align*}
  \frac{d}{dx}\big[f_1'(x)f_2(x) - f_1(x) f_2'(x) \big]
  &= f_1''(x)f_2(x) + f_1'(x)f_2'(x) - f_1'(x) f_2'(x) - f_1(x) f_2''(x) \\
  &\leq f_1''(x)f_1(x) - f_1(x) f_2''(x) \\
  &\leq f_2''(x) f_1(x) - f_1(x) f_2''(x) \\
  &= 0,
\end{align*}
where the first inequality holds because $f_1(x) \leq f_2(x)$ and $f_1''(x) \leq 0$, and the second inequality holds because $f_1''(x) \leq f_2''(x)$ and $f_1(x)\geq{}0$.
\end{proof}

Applying \cref{lem:func_compare} to $-N_p(v)$ and $-D_p(v)$ gives that 
the minimum of $N_p(v)/D_p(v)$ over $[p-1,0)$ will be achieved at $v = p-1$.
Thus,
\begin{align*}
  &\hspace{-1em}\frac{\log(p) + \log(1-p-v) - \log(p-v) - \log(1-p-2v)}{\log(p) + \log(1-p-v) - \log(p-v) - \log(1-p) + \frac{2v}{1-p}} \\
  &\geq \frac{\log(p) + \log(1-p-(p-1)) - \log(p-(p-1)) - \log(1-p-2(p-1))}{\log(p) + \log(1-p-(p-1)) - \log(p-(p-1)) - \log(1-p) + \frac{2(p-1)}{1-p}} \\
  &= \frac{\log(p) + \log(2-2p) - \log(1) - \log(3-3p)}{\log(p) + \log(2-2p) - \log(1) - \log(1-p) - 2} \\
  &= \frac{\log(p) + \log(2) - \log(3)}{\log(p)+\log(2)-2}.
\end{align*}
Since $\log(2) < \log(3) < 2$, and since $\log(p)\leq{}0$, the
expression above decreases as $p$ increases, which means the minimum
is achieved at $p=1$, so we conclude that
\begin{align*}
  \frac{\log(p) + \log(1-p-v) - \log(p-v) - \log(1-p-2v)}{\log(p) + \log(1-p-v) - \log(p-v) - \log(1-p) + \frac{2v}{1-p}} 
  \geq \frac{\log(3) - \log(2)}{2 - \log(2)}.
\end{align*}
This means that for all $p \in (0,1)$,
\begin{align*}
  \lambda \leq \frac{\log(3) - \log(2)}{2 - \log(2)} \implies \forall \ v \in [p-1,0],\;\expfunc'_{p,\lambda}(v) \geq 0.
\end{align*}

We now consider the case when $v \in (0,p]$, which follows from the same logic as the argument for $v < 0$.
\begin{align*}
  \expfunc_{p,\lambda}(v) = p \left(1 + \frac{v}{p} \right)^\lambda \exp\left\{-\left(\frac{2\lambda v}{p} \right) \right\} + (1-p) \left(1 + \frac{v}{1-p} \right)^\lambda.
\end{align*}
Thus,
\begin{align*}
  \expfunc'_{p,\lambda}(v)
  &= p\lambda \left(1+\frac{v}{p} \right)^{\lambda-1} \left(\frac{1}{p}\right) \exp\left\{-\left(\frac{2\lambda v}{p} \right) \right\}
    - p \left(1 + \frac{v}{p} \right)^\lambda \exp\left\{-\left(\frac{2\lambda v}{p} \right) \right\} \left(\frac{2\lambda}{p} \right) \\
  & \qquad + (1-p)\lambda \left(1 + \frac{v}{1-p} \right)^{\lambda-1} \left(\frac{1}{1-p}\right) \\
  &= \lambda \left[\left(\frac{p+v}{p} e^{\frac{-2v}{p}}\right)^\lambda \left(\frac{p}{p+v} - 2 \right) + \left(\frac{1-p+v}{1-p} \right)^{\lambda-1} \right].
\end{align*}
That is,%
\begin{align*}
  &\expfunc'_{p,\lambda}(v) \leq 0 \\
&\iff  \lambda \left[\left(\frac{p+v}{p} e^{\frac{-2v}{p}}\right)^\lambda \left(\frac{p}{p+v} - 2 \right) + \left(\frac{1-p+v}{1-p} \right)^{\lambda-1} \right] \leq 0\\
  &\iff\left(\frac{1-p+v}{1-p} \right)^{\lambda-1} \leq \left(\frac{p+v}{p} e^{\frac{-2v}{p}}\right)^\lambda \left(\frac{p+2v}{p+v} \right)\\
  &\iff\frac{(1-p)(p+v)}{(1-p+v)(p+2v)} \leq \left(\frac{(p+v)(1-p)}{p(1-p+v)} e^{\frac{-2v}{p}}\right)^\lambda.
\end{align*}
Now, 
\begin{align*}
  \frac{d}{dv} \frac{(p+v)(1-p)}{p(1-p+v)} e^{\frac{-2v}{p}}
  &= \frac{1-p}{p} \left[\left(\frac{-2}{p}\right) \frac{p+v}{1-p+v} e^{\frac{-2v}{p}} - \frac{p+v}{(1-p+v)^2} e^{\frac{-2v}{p}} + \frac{1}{1-p+v} e^{\frac{-2v}{p}} \right] \\
  &= \frac{1-p}{p(1-p+v)} e^{\frac{-2v}{p}} \left[1 -\frac{p+v}{1-p+v} - \frac{2(p+v)}{p} \right] \\
  &= -\frac{1-p}{p(1-p+v)} e^{\frac{-2v}{p}} \left[1 + \frac{2v}{p} + \frac{p+v}{1-p+v} \right] \\
  &\leq 0.
\end{align*}
So, since this term is decreasing as $v$ increases, and at $v=0$ it takes the value 1, we have $\frac{(p+v)(1-p)}{p(1-p+v)} e^{\frac{-2v}{p}} \leq 1$. Thus, for all $p \in (0,1)$,
\begin{align}
  \expfunc'_{p,\lambda}(v) \leq 0 
  \iff
  \lambda \leq \frac{\log(1-p) + \log(p+v) - \log(1-p+v) - \log(p+2v)}{\log(1-p) + \log(p+v) - \log(1-p+v) - \log(p) - \frac{2v}{p}}.
  \label{eqn:case2_lambda_condition}
\end{align}
By the same argument from the $v < 0$ case applied to $N_p(-v)$ and $D_p(-v)$ defined by the RHS of (\ref{eqn:case2_lambda_condition}), and another application of \cref{lem:func_compare},
we conclude that the minimum of the RHS over $(0,p]$ will be achieved at $v=p$. That is,
\begin{align*}
  &\hspace{-1em} \frac{\log(1-p) + \log(p+v) - \log(1-p+v) - \log(p+2v)}{\log(1-p) + \log(p+v) - \log(1-p+v) - \log(p) - \frac{2v}{p}} \\
  &\geq \frac{\log(1-p) + \log(p+p) - \log(1-p+p) - \log(p+2p)}{\log(1-p) + \log(p+p) - \log(1-p+p) - \log(p) - \frac{2p}{p}} \\
  &= \frac{\log(1-p) + \log(2p) - \log(1) - \log(3p)}{\log(1-p) + \log(2p) - \log(1) - \log(p) - 2} \\
  &= \frac{\log(1-p) + \log(2) - \log(3)}{\log(1-p)+\log(2)-2}.
\end{align*}
Again, since $\log(2) < \log(3) < 2$, this decreases as $p$ decreases, which means the minimum is achieved at $p=0$, so
\begin{align*}
  \frac{\log(1-p) + \log(p+v) - \log(1-p+v) - \log(p+2v)}{\log(1-p) + \log(p+v) - \log(1-p+v) - \log(p) - \frac{2v}{p}}
  \geq \frac{\log(3) - \log(2)}{2 - \log(2)}.
\end{align*}
This implies that for all $p \in (0,1)$,
\begin{align*}
  \lambda \leq \frac{\log(3) - \log(2)}{2 - \log(2)} \implies \forall \ v \in [0,p], \expfunc'_{p,\lambda}(v) \leq 0.
\end{align*}

Combining these results, we have that for all $p \in [0,1]$,
\icml{
\begin{align*}
  \qquad
  \lambda \leq \frac{\log(3) - \log(2)}{2 - \log(2)} \implies \forall \ v \in [p-1,p],\;\;
  \EE_{y \sim p} \exp\left\{\lambda \eta(p,y) v - \lambda \abs{\eta(p,y) v} + \lambda \log\left(1 + \abs{\eta(p,y) v} \right) \right\} \leq 1.
  \ \qed
\end{align*}}
\arxiv{
\begin{align*}
  \lambda \leq \frac{\log(3) - \log(2)}{2 - \log(2)} \implies \forall \ v \in [p-1,p],\;\;
  \EE_{y \sim p} \exp\left\{\lambda \eta(p,y) v - \lambda \abs{\eta(p,y) v} + \lambda \log\left(1 + \abs{\eta(p,y) v} \right) \right\} \leq 1.
  \qed
\end{align*}}

\icml{\section{Proof of \cref{prop:compare}}}
\arxiv{\section{Proof of Proposition~\ref{prop:compare}}}
\label{prf:compare}

For case (i), 
let $\seqentropy{\F}{\gamma}{n}{\infty} = \Theta(\log(1/\gamma))$. Taking $\gamma = 1/n$ gives $\newbound{\F} = \Theta(\log(n))$, which is known to be optimal \citep[see, e.g.,][]{rissanen1996fisher}.

To handle the remaining two cases, we need to optimize $\oldbound{\F}$ for each sequential entropy specification by finding the values of $\gamma,\delta,$ and $\alpha$ that minimize the order of the largest term. Our strategy is to plug in a specific instance of these three parameters, and then show that changing them in any way will lead to an increase in the order of the largest term. We observe the following. For any $p>0$, when $\seqentropy{\F}{\gamma}{n}{\infty} = \Theta(\gamma^{-p})$, our bound is simple to optimize. The optimal parametrization is $\gamma = n^{\vphantom{1}-\frac{1}{p+1}}$, which gives $\newbound{\F} = \Theta(n^{\vphantom{1}\frac{p}{p+1}})$.

Also, for any $p > 0$, we see that (\ref{eq:oldbound}) becomes
\begin{align}
  \oldbound{\F}
  = \inf_{\overset{\gamma \geq \alpha > 0}{\delta > 0}} \tilde\Theta\Big(\frac{\alpha n}{\delta} + \sqrt{\frac{n}{\delta}} \frac{2}{2-p} \Big[\gamma^{\frac{2-p}{2}} - \alpha^{\frac{2-p}{2}} \Big]
  + \frac{1}{\delta} \frac{1}{1-p} \Big[\gamma^{1-p} - \alpha^{1-p} \Big] + \gamma^{-p} + n\delta \Big).
\label{eq:oldbound_expanded}
\end{align}
Of course, when $p = 1$ or $p = 2$ the second and first integrals respectively become $\log(1/\gamma) - \log(1/\alpha)$ rather than $0/0$. We will first consider when $p \notin \{1,2\}$, and then observe our result still holds for these specific cases. 

Next, observe that for all $p>0$, (\ref{eq:oldbound_expanded}) is convex with respect to $\alpha$. Then, differentiating with respect to $\alpha$ and setting it equal to zero (ignoring constants) gives
\begin{align*}
  \frac{n}{\delta} - \sqrt{\frac{n}{\delta}}\alpha^{-p/2} - \frac{1}{\delta}\alpha^{-p} &\approx 0, \intertext{ which we can simplify to}
  n - \sqrt{n\delta}\alpha^{-p/2} - \alpha^{-p} &\approx 0.
\end{align*}
Solving this quadratic reveals that, up to constants, $\alpha = n^{-\frac{1}{p}}$, so we only have to optimize over $\gamma$ and $\delta$. Plugging this into (\ref{eq:oldbound_expanded}), we get
\begin{align}
  \oldbound{\F}
  = \inf_{\gamma \geq n^{-\scriptscriptstyle\frac{1}{p}}, \ \delta>0}  \tilde\Theta\prn*{
  \frac{n^{\vphantom{1}\frac{p-1}{p}}}{\delta}
  + \sqrt{\frac{n}{\delta}} \frac{2}{2-p} 
    \left[\gamma^{\frac{2-p}{2}} - n^{\frac{p-2}{2p}} \right]
  + \frac{1}{\delta} \frac{1}{1-p} 
    \left[\gamma^{1-p} - n^{\vphantom{1}\frac{p-1}{p}} \right]
  + \gamma^{-p} 
  + n\delta
}.
\label{eq:oldbound_expanded2}
\end{align}

We now turn to proving statements (ii) and (iii).

(ii)
If $0 < p < 1$, taking $\gamma = n^{-\frac{1}{p+1}}$ and $\delta = n^{-\frac{1}{p+1}}$ gives
\begin{align}
  \oldbound{\F}
  &= \tilde\Theta\Big(
  n^{\vphantom{1}\frac{p^2+p-1}{p(p+1)}}
  + \frac{2}{2-p} 
    \Big[n^{\vphantom{1}\frac{p}{p+1}} - n^{\vphantom{1}\frac{2p^2+p-2}{2p(p+1)}}\Big] 
  + \frac{1}{1-p} 
    \Big[n^{\vphantom{1}\frac{p}{p+1}} - n^{\vphantom{1}\frac{p^2+p-1}{p(p+1)}} \Big]
  + n^{\vphantom{1}\frac{p}{p+1}}
  + n^{\vphantom{1}\frac{p}{p+1}}
  \Big) \nonumber \\
  &= \tilde\Theta\Big(n^{\vphantom{1}\frac{p}{p+1}}\Big).
\label{eq:dylan_small_p}
\end{align}
We need to show that (\ref{eq:dylan_small_p}) is the optimal polynomial dependence on $n$ for $\oldbound{\F}$ when $0 < p < 1$.

First, observe that when $p < 1$, $\vphantom{1}\frac{2p^2+p-2}{2p(p+1)} \leq \vphantom{1}\frac{p^2+p-1}{p(p+1)} < \vphantom{1}\frac{p}{p+1}$, so the negative terms are not cancelling all of the higher order terms. Also, since we require $\alpha \leq \gamma$, the negative (third and fifth) terms can at most cancel the second and fourth terms. Now, suppose that the highest order exponent $\vphantom{1}\frac{p}{p+1}$ could be lowered. This would require lowering the polynomial dependence on $n$ for the seventh term, which corresponds to $n\delta$ in (\ref{eq:oldbound_expanded2}). Consequently, this would require $\delta = n^{-\frac{1}{p+1}-\beta}$ for some $\beta>0$. We would then obtain
\icml{
\begin{align*}
  \oldbound{\F}
  &= \tilde\Theta\Big(
  n^{\vphantom{1}\frac{p^2+p-1}{p(p+1)}+\beta}
  + \frac{2}{2-p} 
    \Big[n^{\vphantom{1}\frac{p}{p+1}+\beta/2} - n^{\vphantom{1}\frac{2p^2+p-2}{2p(p+1)}+\beta/2}\Big] 
  + \frac{1}{1-p} 
    \Big[n^{\vphantom{1}\frac{p}{p+1}+\beta} - n^{\vphantom{1}\frac{p^2+p-1}{p(p+1)}+\beta} \Big]
  + n^{\vphantom{1}\frac{p}{p+1}}
  + n^{\vphantom{1}\frac{p}{p+1}-\beta}
  \Big).
\end{align*}}
\arxiv{
\begin{align*}
  \oldbound{\F}
  &= \tilde\Theta\Big(
  n^{\vphantom{1}\frac{p^2+p-1}{p(p+1)}+\beta}
  + \frac{2}{2-p} 
    \Big[n^{\vphantom{1}\frac{p}{p+1}+\frac{\beta}{2}} - n^{\vphantom{1}\frac{2p^2+p-2}{2p(p+1)}+\frac{\beta}{2}}\Big] 
  + \frac{1}{1-p} 
    \Big[n^{\vphantom{1}\frac{p}{p+1}+\beta} - n^{\vphantom{1}\frac{p^2+p-1}{p(p+1)}+\beta} \Big]
  + n^{\vphantom{1}\frac{p}{p+1}}
  + n^{\vphantom{1}\frac{p}{p+1}-\beta}
  \Big).
\end{align*}}
The second and fourth terms now have increased in order, and can only be lowered by taking $\gamma = n^{-\frac{1}{p+1}-\beta'}$ for some $\beta' > 0$. This results in
\begin{align*}
  \oldbound{\F}
  &= \tilde\Theta\Big(
  n^{\vphantom{1}\frac{p^2+p-1}{p(p+1)}+\beta}
  + \frac{2}{2-p} 
    \Big[n^{\vphantom{1}\frac{p}{p+1}+\beta/2-\beta'(\frac{2-p}{2})} - n^{\vphantom{1}\frac{2p^2+p-2}{2p(p+1)}+\beta/2}\Big] 
  + \frac{1}{1-p} 
    \Big[n^{\vphantom{1}\frac{p}{p+1}+\beta-\beta'(1-p)} - n^{\vphantom{1}\frac{p^2+p-1}{p(p+1)}+\beta} \Big] \\
  & \qquad~~~ + n^{\vphantom{1}\frac{p}{p+1}+\beta'p}
  + n^{\vphantom{1}\frac{p}{p+1}-\beta}
  \Big).
\end{align*}
However, now the sixth term has increased in order. So, we conclude that the exponent $\frac{p}{p+1}$ cannot be lowered, and thus (\ref{eq:dylan_small_p}) is the optimal polynomial dependence on $n$ for $\oldbound{\F}$ when $0 < p < 1$.

(iii)
If $1 < p < 2$, taking $\gamma = n^{-\frac{2p+1}{2p(2+p)}}$ and $\delta = n^{-\frac{1}{2p}}$ gives
\begin{align}
  \oldbound{\F}
  &= \tilde\Theta\Big(
  n^{\vphantom{1}\frac{2p-1}{2p}}
  + \frac{2}{2-p} 
    \Big[n^{\vphantom{1}\frac{2p+1}{2(2+p)}} - n^{\vphantom{1}\frac{4p-3}{4p}}\Big] 
  + \frac{1}{p-1} 
    \Big[-n^{\vphantom{1}\frac{2p^2+1}{2p(2+p)}} + n^{\vphantom{1}\frac{2p-1}{2p}} \Big]
  + n^{\vphantom{1}\frac{2p+1}{2(2+p)}}
  + n^{\vphantom{1}\frac{2p-1}{2p}}
  \Big) \nonumber \\
  &= \tilde\Theta\Big(n^{\vphantom{1}\frac{2p-1}{2p}}\Big).
\label{eq:dylan_medium_p}
\end{align}
 We now need to show that (\ref{eq:dylan_medium_p}) is the optimal polynomial dependence on $n$ for $\oldbound{\F}$ when $1 < p < 2$. Our argument is very similar to the argument we gave for $0 < p < 1$.

First, observe that when $1 < p < 2$, $\vphantom{1}\frac{4p-3}{4p}$ and $\vphantom{1}\frac{2p^2+1}{2p(2+p)}$ are both less than $\vphantom{1}\frac{2p-1}{2p}$, so the negative terms are not cancelling all of the higher order terms. Further, since we require $\alpha \leq \gamma$, this means we require $\gamma^{1-p} \leq \alpha^{1-p}$ for $p>1$, so at most the third term can cancel the second term and the fourth term can cancel the fifth term.

Now, suppose that the exponent $\vphantom{1}\frac{2p-1}{2p}$ could be lowered. This would again require lowering the polynomial dependence on $n$ for the seventh term, so would require $\delta = n^{-\frac{1}{2p}-\beta}$ for some $\beta>0$. This leads to
\icml{
\begin{align*}
  \oldbound{\F}
  &= \tilde\Theta\Big(
  n^{\vphantom{1}\frac{2p-1}{2p}+\beta}
  + \frac{2}{2-p} 
    \Big[n^{\vphantom{1}\frac{2p+1}{2(2+p)}+\beta/2} - n^{\vphantom{1}\frac{4p-3}{4p}+\beta/2}\Big] 
  + \frac{1}{p-1} 
    \Big[-n^{\vphantom{1}\frac{2p^2+1}{2p(2+p)}+\beta} + n^{\vphantom{1}\frac{2p-1}{2p}+\beta} \Big]
  + n^{\vphantom{1}\frac{2p+1}{2(2+p)}}
  + n^{\vphantom{1}\frac{2p-1}{2p}-\beta}
  \Big).
\end{align*}}
\arxiv{
\begin{align*}
  \hspace{-2pt}\oldbound{\F}
  &= \tilde\Theta\Big(
  n^{\vphantom{1}\frac{2p-1}{2p}+\beta}\hspace{-1pt}
  + \hspace{-1pt}\frac{2}{2-p} 
    \Big[n^{\vphantom{1}\frac{2p+1}{2(2+p)}+\frac{\beta}{2}} - n^{\vphantom{1}\frac{4p-3}{4p}+\frac{\beta}{2}}\Big] \hspace{-1pt}
  + \hspace{-1pt}\frac{1}{p-1} 
    \Big[-n^{\vphantom{1}\frac{2p^2+1}{2p(2+p)}+\beta} + n^{\vphantom{1}\frac{2p-1}{2p}+\beta} \Big]
  + n^{\vphantom{1}\frac{2p+1}{2(2+p)}}
  + n^{\vphantom{1}\frac{2p-1}{2p}-\beta}
  \Big).
\end{align*}}
The only remaining way to reduce the order is to set $\gamma = n^{-\frac{2p+1}{2p(2+p)}-\beta'}$ for some $\beta' > 0$, which leads to
\begin{align*}
  \oldbound{\F}
  &= \tilde\Theta\Big(
  n^{\vphantom{1}\frac{2p-1}{2p}+\beta}
  + \frac{2}{2-p} 
    \Big[n^{\vphantom{1}\frac{2p+1}{2(2+p)}+\beta/2-\beta'(\frac{2-p}{2})} - n^{\vphantom{1}\frac{4p-3}{4p}+\beta/2}\Big] 
  + \frac{1}{p-1} 
    \Big[-n^{\vphantom{1}\frac{2p^2+1}{2p(2+p)}+\beta+\beta'(p-1)} + n^{\vphantom{1}\frac{2p-1}{2p}+\beta} \Big] \\
  & \qquad~~~ + n^{\vphantom{1}\frac{2p+1}{2(2+p)}+\beta'p}
  + n^{\vphantom{1}\frac{2p-1}{2p}-\beta}
  \Big).
\end{align*}
Thus, the first term has increased in order, and as argued cannot be cancelled by either of the negative terms, so we conclude (\ref{eq:dylan_medium_p}) is the optimal polynomial dependence on $n$ for $\oldbound{\F}$ when $1 < p < 2$.

Otherwise, if $p > 2$, taking $\gamma = 1$ and $\delta = n^{-\frac{1}{2p}}$ gives
\begin{align}
  \oldbound{\F}
  &= \tilde\Theta\Big(
  n^{\vphantom{1}\frac{2p-1}{2p}}
  + \frac{2}{p-2} 
    \Big[-n^{\vphantom{1}\frac{2p+1}{4p}} + n^{\vphantom{1}\frac{4p-3}{4p}}\Big] 
  + \frac{1}{p-1} 
    \left[-n^{\vphantom{1}\frac{1}{2p}} + n^{\vphantom{1}\frac{2p-1}{2p}} \right]
  + 1
  + n^{\vphantom{1}\frac{2p-1}{2p}}
  \Big) \nonumber \\
  &= \tilde\Theta\Big(n^{\vphantom{1}\frac{2p-1}{2p}}\Big).
\label{eq:dylan_high_p}
\end{align}
The argument that (\ref{eq:dylan_high_p}) is the optimal polynomial dependence on $n$ for $\oldbound{\F}$ when $p>2$ follows from the same logic as when $1 < p < 2$. The only difference is that now we observe requiring $\alpha \leq \gamma$ means both $\gamma^{\frac{2-p}{2}} \leq \alpha^{\frac{2-p}{2}}$ and $\gamma^{1-p} \leq \alpha^{1-p}$. Then, any adjustment of $\delta$ will force either the first or seventh term to increase in order, and no adjustment of $\gamma$ can cause one of the negative terms to cancel this.

Finally, when $p \in \{1,2\}$, the logic is preserved since we still require $\alpha \leq \gamma$ and we are already ignoring $\polylog(n)$ factors by using $\tilde\Theta$. Thus, all cases have been considered, and dividing $\Theta\big(n^{\frac{p}{p+1}}\big)$ by the respective optimizations of (\ref{eq:oldbound}) gives the desired result.
\icml{\qed}
\arxiv{\qedarxiv}

\section{Additional Details for Proof of Theorem \ref{thm:lower}}
\label{prf:lower}

\subsection{Proof of Lemma \ref{lem:regret_risk}}
\label{prf:regret_risk}

The argument proceeds using the standard online-to-batch conversion
argument \citep[see, e.g.,][]{cesabianchi04online}.
First, for any class $\F$, we can rewrite the minimax regret
as
\begin{align*}
  \minimax{\F}
  = \sup_{x_1 \in \cX} \inf_{\pred_1 \in [0,1]} \sup_{p_1 \in [0,1]} \mathop{\EE}_{y_1 \sim p_1} \cdots \sup_{x_n \in \cX} \inf_{\pred_n \in [0,1]} \sup_{p_n \in [0,1]} \mathop{\EE}_{y_n \sim p_n} \regret{\hat{\vect{p}}}{\F}{\vect{x},\vect{y}}.
\end{align*}
Next, we observe that using the prediction strategy notation $\hat{f}$
described in \pref{sec:lower-proof}, this is equal to 
\begin{align*}
  \minimax{\F}
   =
   \inf_{\hat f}
  \sup_{x_1 \in \cX} \sup_{p_1 \in [0,1]} \mathop{\EE}_{y_1 \sim p_1} \cdots \sup_{x_n \in \cX} \sup_{p_n \in [0,1]} \mathop{\EE}_{y_n \sim p_n} \regret{\hat{\vect{f}} \circ \vect{x}}{\F}{\vect{x},\vect{y}}.
\end{align*}
Then, since the adversary is free to choose the contexts and
observations an any way to maximize the expected regret, we can move
to a lower bound by forcing them to draw $(x_t,y_t)$ i.i.d. from a
joint distribution $\cD\in\cP$. Thus,
\begin{align*}
  \minimax{\F}
  \geq 
  \inf_{\hat f}
  \sup_{\dist \in \borel} %
  \mathop{\EE}_{(x_{1:n},y_{1:n}) \sim \dist} \regret{\hat{\vect{f}} \circ \vect{x}}{\F}{\vect{x},\vect{y}}.
\end{align*}
Then, expanding the definition of regret
and applying Jensen's inequality to the $\sup_{f \in \F}$ gives
\begin{align*}
  \minimax{\F}
  \geq 
  \inf_{\hat f}
  \sup_{\dist \in \borel} \sup_{f \in \F}
  \mathop{\EE}_{(x_{1:n}, y_{1:n}) \sim \dist^{\otimes n}} \sn \Big[\logloss(\hat f_t(x_t),y_t) - \logloss(f(x_t),y_t) \Big].
\end{align*}
Next, rewriting the sum over $t$ as an average gives
\begin{align*}
  \minimax{\F}
  \geq 
  n \cdot
  \inf_{\hat f}
  \sup_{\dist \in \borel} \sup_{f \in \F}
  \mathop{\EE}_{t \in [n]} \mathop{\EE}_{(x_{1:t}, y_{1:t}) \sim \dist^{\otimes t}} \Big[\logloss(\hat f_t(x_t),y_t) - \logloss(f(x_t),y_t) \Big].
\end{align*}
Clearly, for any distribution $\dist \in \borelF$,
\begin{align*}
  \sup_{f \in \F} \mathop{\EE}_{t \in [n]} \mathop{\EE}_{(x_{1:t}, y_{1:t}) \sim \dist^{\otimes t}} \Big[- \logloss(f(x_t),y_t)\Big]
  = - \inf_{f \in \F} \mathop{\EE}_{(x,y) \sim \dist} \Big[\logloss(f(x),y)\Big]
  = - \mathop{\EE}_{(x,y) \sim \dist}\Big[\logloss(\optf(x),y)\Big].
\end{align*}
Further, by convexity of $\logloss$ in the first argument, for any $\hat f$ and $\dist$ it holds that
\begin{align*}
  &\hspace{-1em}\mathop{\EE}_{t \in [n]} \mathop{\EE}_{(x_{1:t}, y_{1:t}) \sim \dist^{\otimes t}} \Big[\logloss(\hat f_t(x_t),y_t) \Big] \\
  &= \mathop{\EE}_{(x_{1:n-1}, y_{1:n-1}) \sim \dist^{\otimes n-1}} \mathop{\EE}_{(x,y)\sim\dist} \mathop{\EE}_{t \in [n]} \Big[\logloss(\hat f_t(x),y) \Big] \\
  &\geq \mathop{\EE}_{(x_{1:n-1}, y_{1:n-1}) \sim \dist^{\otimes n-1}} \mathop{\EE}_{(x,y)\sim\dist} \Big[\logloss(\bar f_n(x),y) \Big],
\end{align*}
where we've used that the function $\hat f_t$ is determined by $(x_{1:t-1}, y_{1:t-1})$. 
Thus, the minimax regret is further lower bounded by
\begin{align*}
  \minimax{\F}
  \geq n \cdot{} 
  \inf_{\hat f}
  \sup_{\dist \in \borelF} \EE
    \left[\logloss(\bar f_{n}(x),y) - \logloss(\optf(x),y) \right],
\end{align*}
where $\EE$ is shorthand for the expectation over $(x_{1:n-1}, y_{1:n-1}) \sim \dist^{\otimes n-1}$ and $(x,y) \sim \dist$ with $\dist$ satisfying $y \lvert x \sim \berndist(\optf(x))$. Finally, since the difference of log loss is exactly the KL divergence conditional on an input $x$,
\begin{align}
  \minimax{\F}
  \geq n \cdot{} \inf_{\hat f} \sup_{\dist \in \borelF} \EE \Big[\KL{\berndist(\optf(x))}{\berndist(\bar f_{n}(x))}\Big],
\label{eqn:lb-KL-risk}
\end{align}
where $\berndist(p)$ denotes the Bernoulli distribution with mean
$p$. The result then follows by observing that the RHS of
(\ref{eqn:lb-KL-risk}) upper bounds the $\inf$ over the risk of all $\hat f_n$, since this includes the possibility of $\hat f_n = \bar f_n$.
\icml{\qed}
\arxiv{\qedarxiv}

\subsection{Additional Lemmas}\label{prf:lower_additional}

\begin{lemma}\label{lem:kl_eps}
For any $0 < \eps \leq 1/2$ and $q \in [0,1]$,
\begin{align*}
  \KL{\eps}{q} \geq \frac{\eps}{4}\1\{q \geq 2\eps\} + \frac{\eps}{6} \1\{ q \leq \eps/2\}.
\end{align*}
\end{lemma}
\begin{proof}
Fix $0 < \eps \leq 1/2$. First, observe that for all $p,q \in [0,1]$,
\begin{align*}
  (1-p)\log\left(\frac{1-p}{1-q} \right) = (1-p)\log\left(1 + \frac{q-p}{1-q} \right) \geq q-p,
\end{align*}
where we have used that $\log(1+x) \geq \frac{x}{1+x}$ for all $x \geq -1$. Then,
\begin{align*}
  \KL{\eps}{2\eps} 
  = \eps \log\left(\frac{\eps}{2\eps} \right) + (1-\eps)\log\left(\frac{1-\eps}{1-2\eps} \right)
  \geq -\eps\log(2)  + 2\eps - \eps
  = \eps(1 - \log(2))
  \geq \frac{\eps}{4}.
\end{align*}
Similarly,
\begin{align*}
  \KL{\eps}{\eps/2} 
  = \eps \log\left(\frac{\eps}{\eps/2} \right) + (1-\eps)\log\left(\frac{1-\eps}{1-\eps/2} \right)
  \geq \eps\log(2)  + \eps/2 - \eps
  = \eps(\log(2) - 1/2)
  \geq \frac{\eps}{6}.
\end{align*}

Next, consider the map $f(q) = \KL{\eps}{q}$. By definition,
\begin{align*}
  f'(q) = -\frac{\eps}{q} + \frac{1-\eps}{1-q} = \frac{q-\eps}{q(1-q)}.
\end{align*}
That is, $f'(q) \geq 0$ when $q \geq \eps$ and $f'(q) < 0$ when $q < \eps$, so if $q \geq 2\eps$ then $\KL{\eps}{q} \geq \eps/4$ and if $q \leq \eps/2$ then $\KL{\eps}{q} \geq \eps/6$.
\end{proof}

\begin{lemma}\label{lem:lipschitz}
For any $p \in \Nats$, let $\F = \{f:[0,1]^p \to \Reals \mid \forall x,y\in[0,1]^p \ \abs{f(x) - f(y)} \leq \Norm{x - y}_\infty \}$ denote the class of 1-Lipschitz functions. Then, $\seqentropy{\F}{\gamma}{n}{\infty} = \Theta(\gamma^{-p})$.
\end{lemma}
\begin{proof}
Since sequential entropy is never larger
than the uniform metric entropy,
the upper bound follows from a standard argument bounding the
uniform entropy for Lipschitz functions \citep[see, e.g., Example~5.10 of][]{wainwright19book}. 

Then, we observe that sequential entropy is lower bounded by empirical entropy, 
\begin{align*}
  \empentropy{\F}{\gamma}{n}{\infty} = \sup_{x_{1:n}} \log(\unifcover{\F\rvert_{x_{1:n}}}{\gamma}{\infty}),
\end{align*}
where $\unifcover{\F\rvert_{x_{1:n}}}{\gamma}{\infty}$ denotes the $L_{\infty}$
covering number of the restriction of $\F$ to $x_{1:n} =
(x_1,\dots,x_n)$. This bound holds since any dataset $x_{1:n}$ can be
turned into a tree $\vect{z}$ of depth $n$ by taking $\node{z} = x_t$
for all $y \in \{0,1\}^n$. Divide $[0,1]^p$ into $\gamma^{-p}$
equally spaced intervals and once again use the usual packing
construction for Lipschitz functions from \citet{wainwright19book}, which
has size $2^{(\gamma/c)^{-p}}$ for some numerical constant $c>0$. Then, once $n > \gamma^{-p}$, this
construction serves as a packing on $\F\rvert_{x_{1:n}}$ when
$x_{1:n}$ contains the grid coordinates, so
$\empentropy{\F}{\gamma}{n}{\infty} \geq \Omega(\gamma^{-p})$.
\end{proof}

\end{document}